\documentclass[letterpaper]{article} 
\usepackage{aaai2026}  
\usepackage{times}  
\usepackage{helvet}  
\usepackage{courier}  
\usepackage[hyphens]{url}  
\usepackage{graphicx} 
\usepackage{natbib}  
\usepackage{caption} 
\usepackage{algorithm}
\usepackage{algorithmic}
\usepackage{amsmath}
\usepackage{dsfont}
\usepackage{bm}

\usepackage{booktabs}
\usepackage{siunitx}
\usepackage{makecell}

\usepackage{amssymb}
\usepackage{multirow}
\usepackage{graphicx}
\usepackage{subcaption}
\usepackage{dsfont}
\usepackage[dvipsnames]{xcolor}
\usepackage[colorinlistoftodos]{todonotes}

\usepackage{listings}
\usepackage{xcolor}
\usepackage[scaled=0.95]{inconsolata}

\definecolor{probnum}{RGB}{150,20,0}   
\definecolor{keywords}{RGB}{100,0,0}   
\definecolor{codebg}{RGB}{250,250,250} 

\lstdefinestyle{prism}{
  basicstyle=\ttfamily\footnotesize\linespread{0.9}\selectfont, 
  keywordstyle=\color{keywords}\bfseries,
  commentstyle=\color{gray!75}\itshape,        
  stringstyle=\color{green!40!black},
  columns=fullflexible,
  keepspaces=true,
  showstringspaces=false,
  numbers=none,
  morekeywords={mdp,module,endmodule,init},
  mathescape=true,
  basewidth=0.5em,                            
  backgroundcolor=\color{lightred!15!white},
  frame=single,
  rulecolor=\color{gray!75},
  framesep=2pt,
  framerule=0.3pt,
  xleftmargin=2pt,
  framexleftmargin=2pt,
  aboveskip=2pt,
  belowskip=1pt,
  breaklines=true,
  breakatwhitespace=false,
  emph={up,stay},
  emphstyle=\color{teal!70!black}\bfseries,
  emph={[2]x,y,x',y',ix,iy},
  emphstyle=[2]\bfseries,
  literate=*
    {0.1}{{\color{probnum}\textbf{0.1}}}{3}
    {0.3}{{\color{probnum}\textbf{0.3}}}{3}
    {0.4}{{\color{probnum}\textbf{0.4}}}{3}
    {0.6}{{\color{probnum}\textbf{0.6}}}{3}
    {0.7}{{\color{probnum}\textbf{0.7}}}{3}
    {0.8}{{\color{probnum}0.8}}{3}
    {0.9}{{\color{probnum}\textbf{0.9}}}{3}
    {1.0}{{\color{probnum}1.0}}{3},
}

\definecolor{midgray}{gray}{0.375}

\newcommand{\argsup}{\mathop{\mathrm{argsup}}\limits}

\usepackage{mathtools}
\usepackage{pdflscape} 

\usepackage{amsthm}
\newtheorem{theorem}{Theorem}    

\theoremstyle{definition}

\newtheorem{example}[theorem]{Example}

\theoremstyle{remark}

\usepackage{newfloat}
\usepackage{listings}
\DeclareCaptionStyle{ruled}{labelfont=normalfont,labelsep=colon,strut=off} 
\floatstyle{ruled}
\newfloat{listing}{tb}{lst}{}
\floatname{listing}{Listing}
 \usepackage[colorlinks=true, linkcolor=blue, urlcolor=blue, citecolor=blue]{hyperref}
 \usepackage{navigator} 
\nocopyright 
\def\techreport{}

\definecolor{lightred}{RGB}{241, 225, 222}
\definecolor{color1}{rgb}{0.1,0.498039215686275,0.9549019607843137}
\definecolor{alizarin}{rgb}{0.82, 0.1, 0.26}
\definecolor{antiquewhite}{rgb}{0.98, 0.92, 0.84}
\definecolor{azure}{rgb}{0.94, 1.0, 1.0}
\definecolor{offwhite}{rgb}{0.98, 0.95, 0.95}
\definecolor{pigment}{rgb}{0.2, 0.2, 0.6}

\title{Efficient Solution and Learning of Robust Factored MDPs}
\author {
    Yannik Schnitzer, Alessandro Abate, David Parker
}
\affiliations{
    University of Oxford\\
    Department of Computer Science\\
    \{yannik.schnitzer,alessandro.abate,david.parker\}@cs.ox.ac.uk
}

\usepackage{bibentry}

\begin{document}

\maketitle

\begin{abstract}
Robust Markov decision processes (r-MDPs) extend MDPs by explicitly modelling epistemic uncertainty about 
transition dynamics. Learning r-MDPs from interactions with an unknown environment enables the synthesis of robust policies with provable (PAC) guarantees on performance, but this can require a large number of sample interactions.
We propose novel methods for solving and learning r-MDPs
based on \emph{factored} state-space representations
that leverage the independence between model uncertainty across system components. Although policy synthesis for factored r-MDPs leads to hard, non-convex optimisation problems, we show how to reformulate these into tractable linear programs. Building on these, we also propose methods to learn factored model representations directly.
Our experimental results show that exploiting factored structure can yield dimensional gains in sample efficiency, producing more effective robust policies with tighter performance guarantees than state-of-the-art methods.

\end{abstract}


\section{Introduction}
\label{sec:intro}

\emph{Markov decision processes} (MDPs) are the standard modelling framework for sequential decision-making under uncertainty. However, real-world dynamics are often complex and not fully known. In safety-critical settings, it is therefore essential to reason about \textit{epistemic uncertainty}, due to incomplete knowledge of the environment, and to construct \textit{robust} policies that provide provable performance guarantees on the unknown environment they operate in.

\emph{Robust Markov decision processes} (r-MDPs)~\cite{DBLP:journals/mor/WiesemannKR13}
extend MDPs by not requiring every transition probability to be known precisely but only restricting them to lie in a given \emph{uncertainty set}. These uncertainty sets are typically derived from data, e.g., observed interactions with the unknown system, as in reinforcement learning (RL). Learning for r-MDPs, however, does not optimise for expected performance alone; rather, it enables the synthesis of policies that are robust with respect to the current epistemic uncertainty in the transition dynamics and provides provable \emph{Probably Approximately Correct} (PAC) guarantees on performance with high confidence~\cite{DBLP:conf/icml/StrehlL05,DBLP:conf/nips/SuilenS0022}.

Unlike robust RL approaches, that often focus on heuristic or empirical training for difficult scenarios~\citep{DBLP:conf/nips/MorimotoA02,DBLP:conf/icml/PintoDSG17}, r-MDP learning operates on explicit uncertainty sets learned from data and yields formal anytime guarantees on worst-case performance under the true but unknown transition model.

A practical limitation of r-MDP learning and policy synthesis, however, is that, to achieve high-confidence performance guarantees, the overall confidence level must be distributed across all transition distributions~\citep{DBLP:conf/icml/StrehlL05} or individual transition probabilities~\cite{DBLP:conf/nips/SuilenS0022} being learnt. In large-scale environments, this enforces stringent confidence requirements, requiring a high number of samples to construct tight uncertainty sets that yield effective robust policies with meaningful guarantees.

Many real-world domains come with structural knowledge that permits distinct features of the state space to be modelled independently, giving rise to the model of \emph{factored} MDPs (f-MDPs)~\cite{DBLP:conf/ijcai/KollerP99,DBLP:journals/jair/BoutilierDH99}. RL algorithms have been extended to exploit this factored structure~\cite{DBLP:conf/ijcai/KearnsK99,DBLP:conf/icml/GuestrinPS02,Strehl2007}, often yielding exponential improvements in sample efficiency over learning in the \emph{flat} (non-factored) representation. While these methods come with PAC guarantees, ensuring that a near-optimal policy is learned with high probability in time polynomial in the factored representation, existing work focuses on expected performance and convergence rather than providing
provable guarantees on worst-case performance.

In this work, we introduce a \emph{robust factored MDP} framework, which leverages structural independence to construct uncertainty sets for each state factor rather than for a flat model. We show that robust policy synthesis in this setting leads to intractable non-convex optimisation problems, but that for standard uncertainty classes, such as confidence intervals, $L_{1}$ balls and general polytopes, these problems admit exact convex reformulations. To address the computational challenges of the resulting, potentially exponential constraint sets, we leverage convex relaxations that preserve tight performance guarantees while enabling efficient solution. We show that our method synthesises more effective robust policies with high-confidence performance guarantees that are substantially tighter than those of prior factored MDP learning approaches.  Furthermore, we show that exploiting the factored structure can improve the sample efficiency of robust policy learning by orders of magnitude compared to state‐of‐the‐art methods in flat representations.

\section{Problem Formulation}
\label{sec:prelim}

The set of all probability distributions over a finite set $Y$ is denoted by $\Delta(Y) = \{\mu\colon Y \to [0,1] \mid \sum_{y \in Y} \mu(y) = 1\}$. For convenience, we also represent distributions as vectors in the probability simplex, $(p_1, \dots, p_{|Y|}) \in \Delta_{|Y|}$, where $p_i = \mu(y_i)$ under a fixed ordering of the elements of $Y$.

\subsection{MDPs and Factored MDPs}
A \emph{Markov decision process} (or \emph{MDP}) is a tuple $M = (S, A, T, r)$, where $S$ and $A$ are finite sets of states and actions, $T \colon S \times A \to \Delta(S)$ is a transition probability function, and $r \colon S \times A \to \mathbb{R}$ is a reward function. 
A \emph{policy} is a mapping $\pi \colon (S \times A)^* \times S \to \Delta(A)$ that resolves the non-determinism by selecting a distribution over actions based on the current state and past interactions. The interaction between a policy and an MDP induces infinite sequences (or \emph{paths}) of the form $s^0 a^0 s^1 a^1 \!\dots$, where at each step, the next action is drawn from the distribution assigned by the policy, given the current history prefix, and the next state is drawn from the transition distribution $T(\, \cdot \, | s, a)$.

A \emph{factored MDP} (or \emph{f-MDP}) is an MDP in which states are represented as vectors of $n$ components. 
Each \emph{factor} $i$ (also referred to as a \emph{state variable} or \emph{state marginal}) takes values from a finite domain $\mathcal{D}_i$. Hence, states are tuples $s = (s_1, \dots, s_n)$, with $s_i \in \mathcal{D}_i$.
To capture the (in-)dependence between factors, we adopt the framework of~\citet{Strehl2007}. Given an arbitrary set $\mathcal{I}$ of \emph{dependency identifiers}, a function $D_i \colon S \times A \to \mathcal{I}$ is a \emph{dependency function} for factor $i$. The transition function is then defined as
\begin{equation}
    \label{eq:transfunc}
    T(s' | s, a) = \prod_{i = 1}^{n} P(s_i' | D_i(s, a)),
\end{equation}
where $s_i'$ denotes the $i$-th component of the next state $s'$ 
and each $P(\, \cdot \, | D_i(s, a)) \in \Delta(\mathcal{D}_i)$ specifies the \emph{marginal probability distribution} of the respective factor. 

\begin{example}\label{ex:1}
A classic example of a factored MDP is the System Administrator domain (see \citet{DBLP:journals/jair/GuestrinKPV03} for a detailed description). An administrator controls $n$ machines (the factors), each of which can be either operational or in a failure state. The machines form a network: each machine is connected to a subset of the others and has a small probability of failing at the next step. This probability depends only on whether its neighbouring machines are operational and may increase substantially if one of them fails, but is independent of all other, non-neighbouring machines. 

For machine $i$, the corresponding dependency identifier therefore encodes the current state of machine $i$ together with the states of all its neighbours. The dependency function $D_i$ acts as a projection that extracts these components from the global state $s = (s_1, \dots, s_n)$. If one or more neighbours of machine $i$ are in a failure state, then the marginal probability $P(s_i' = \mathit{fail} \mid D_i(s, a))$ that machine $i$ fails in the next step increases. 
\end{example}

\begin{example}\label{ex:prism}
Another natural and established way of specifying factored MDPs is via the modelling language of the PRISM verification tool~\cite{DBLP:conf/cav/KwiatkowskaNP11}. This is a guarded-command-style language with a modular structure that matches
the factored MDP representation. Each PRISM
\emph{module} encodes the local dynamics of one component and thus corresponds
naturally to a state factor. Its behaviour is given by commands of the form
\[[\texttt{action}]~\mathit{guard} \rightarrow \sum_k p_k : \mathit{update}_k
,\]
where the guard is a predicate over the global state and the updates define a
local (i.e., \emph{marginal}) probability distribution over the module's next state. 
As an illustrative example, we consider the following excerpt from an aircraft collision avoidance model~\citep{kochenderfer2015} written in the PRISM language, with an ownship and an intruder moving on a grid (full details of this domain can be found in Appendix~\ref{app:benchmarks}):
\vspace{6pt}
\begin{lstlisting}[style=prism,mathescape=true]
module ownship
  x : [0..X] init 0   $\color{midgray}\textit{// horizontal position}$
  y : [0..Y] init 0   $\color{midgray}\textit{// altitude}$

  $\color{midgray}\textit{// [up]: dependency on closeness to intruder}$
  [up] $\parallel$(x,y) - (ix,iy)$\parallel_1$ >= 5 ->  0.9 : (y' = y + 1) 
                                 + 0.1 : (y' = y)
  [up] $\parallel$(x,y) - (ix,iy)$\parallel_1$ < 5  ->  0.6 : (y' = y + 1) 
                                 + 0.4 : (y' = y)
  ...
endmodule

module intruder
  ix : [0..IX] init IX $\color{midgray}\textit{// intruder horizontal position}$
  iy : [0..IY] init 0  $\color{midgray}\textit{// intruder altitude}$

  $\color{midgray}\textit{// [up]: no dependency on ownship}$
  [up] iy < 10 -> 0.7 : (iy' = iy + 1) 
                + 0.3 : (iy' = iy)
  ...
endmodule
\end{lstlisting}
\vspace{6pt}
The global state is $s = (x,y,ix,iy)$, which factorises into an ownship factor $(x,y)$
and an intruder factor $(ix,iy)$. For the action \texttt{up}, the ownship module has two
guards: one for states where the intruder is \emph{far}
($\lVert(x,y) - (ix,iy)\rVert_1 \geq 5$) and one where it is \emph{near}. This corresponds to a dependency in which the dynamics of the ownship change with respect to the intruder, for example through altered aerodynamic conditions when the intruder is near, resulting in a reduced success probability of the intended upward motion.

In the factored MDP notation, we capture the two guards for the \texttt{up}
command of the ownship module by introducing two dependency identifiers,
denoted $j_{\mathrm{up}}^{\mathit{far}}$ and $j_{\mathrm{up}}^{\mathit{near}}$.
The dependency function $D_{\mathrm{own}}(s,\texttt{up})$ simply evaluates the
guard predicate $\lVert (x,y) - (ix,iy) \rVert_1 \ge 5$ in state $s$ and returns
$j_{\mathrm{up}}^{\mathit{far}}$ if it holds and $j_{\mathrm{up}}^{\mathit{near}}$
otherwise.
Thus, all global states with the same
near--far relation share the same local transition distribution for the ownship 
$P_{\mathrm{own}}(\cdot \mid D_{\mathrm{own}}(s,\texttt{up}))$.
The intruder command depends only on its own local variables $(ix,iy)$. Accordingly, the dependency function $D_{\mathrm{int}}(s,\texttt{up})$
ignores $(x,y)$ and only inspects $(ix,iy)$. In particular, it is constant over all global states that agree on $(ix,iy)$. 

Under PRISM's synchronous composition semantics, the local
updates combine into the global transition:
\[
\begin{aligned}
T(s' | s,a)
=& P_{\mathrm{own}}\bigl( (x',y') \mid D_{\mathrm{own}}(s,a) \bigr) \\
\quad\cdot & P_{\mathrm{int}}\bigl( (ix',iy') \mid D_{\mathrm{int}}(s,a) \bigr),
\end{aligned}
\]
which is exactly the product form in Eq.~\eqref{eq:transfunc}. In
this way, PRISM modules correspond to factors, guards implement the dependency
identifiers $D_i(s,a)$, and the local updates to module states on the right-hand side of each command specify the
marginal transition distributions $P(\cdot \mid D_i(s,a))$.
\end{example}

\subsection{Robust Factored MDPs}
\emph{Robust factored MDPs} (or \emph{rf-MDPs})~\cite{conf/isipta/DelgadoBCS09,journals/corr/abs-2404-02006} extend factored MDPs to incorporate epistemic uncertainty about transition dynamics. They generalise fixed marginal transition distributions $P(\, \cdot \,| D_i(s, a)) \in \Delta(\mathcal{D}_i)$ to \emph{marginal uncertainty sets} $\mathcal{P}(D_i(s, a)) \subseteq \Delta(\mathcal{D}_i)$. The overall uncertainty set of possible transition distributions  at $(s,a)$ is then defined as: 
\begin{equation}
    \label{eq:transunc}
    \mathcal{T}(s, a) = \bigotimes_{i=1}^n \mathcal{P}(D_i(s, a)),
\end{equation}
where $\otimes$ denotes the outer product (or Kronecker product) of distributions, extended to sets. Specifically, for sets $\mathcal{P} \subseteq \Delta(\mathcal{D}_i)$ and $\mathcal{Q} \subseteq \Delta(\mathcal{D}_j)$, the product is defined as
\begin{equation}
\mathcal{P} \otimes \mathcal{Q} = \{P \otimes Q \mid P \in \mathcal{P},\, Q \in \mathcal{Q}\},
\end{equation}
where for distributions $P = (p_1, \dots, p_m)$ and $Q = (q_1, \dots, q_k)$, their outer product is given by
\begin{equation}
\label{eq:productdists}
(P \otimes Q)_{ij} = p_i \cdot q_j, \quad 1\leq i \leq m, 1 \leq j \leq k.
\end{equation}
Hence, \(\mathcal{T}(s, a)\) comprises all product distributions over the factor-wise uncertainty sets, providing a structured representation of the uncertainty over the full state space \(S\). 

As in standard robust MDPs~\citep{DBLP:journals/ior/NilimG05,DBLP:journals/mor/Iyengar05,DBLP:journals/mor/WiesemannKR13}, rf-MDPs introduce an additional step in the evolution of the process: at a given state $s$, before the next state is determined following the selection of action $a$, an \emph{environment policy}~$\tau$ selects, for each factor $i$, a marginal distribution from the corresponding uncertainty set $\mathcal{P}(D_i(s, a))$. These combine into a product distribution, as per Eq.~\eqref{eq:transfunc}, which lies in the overall uncertainty set $\mathcal{T}(s, a)$ and defines the probability distribution from which the successor state is drawn.

\paragraph{Objectives.} An \emph{objective} is a mapping $R$ that assigns a return to each infinite path $\rho = s^0 a^0 s^1 a^1 \! \dots$ in an rf-MDP $\tilde{M}$. Given a pair of agent and environment policies $\pi$ and $\tau$, we denote by $\mathbb{E}^{\pi,\tau}_{\tilde{M},s}$ the induced expectation over paths starting in state $s$~\citep{DBLP:conf/cdc/WolffTM12}. The \emph{value} of $s$ under $\pi$ and $\tau$ with respect to objective $R$ is defined as
\begin{equation}
    V^{\pi,\tau}_{\tilde{M}}(s) = \mathbb{E}^{\pi,\tau}_{\tilde{M},s} [R]. 
\end{equation}
Unless stated otherwise, our results are agnostic to the specific choice of objective. The most common objective is the discounted cumulative reward:
\begin{equation}
    \label{eq:discountedreward}
    R(\rho) = \sum_{t = 0}^{\infty} \gamma^t r(s^t, a^t),
\end{equation}
for some discount factor $0 < \gamma < 1$. However, our results readily extend to other objectives, such as undiscounted rewards~\cite{DBLP:conf/icml/Schwartz93,DBLP:books/wi/Puterman94,DBLP:conf/aaai/MeggendorferWW25} or reachability goals focussing on the probability of eventually reaching a target set of states, possibly whilst avoiding certain undesirable states.

\paragraph{Robust Values and Policies}
The \emph{optimal robust policy} $\pi^*$ in an rf-MDP $\tilde{M}$ is the policy that achieves, in every state, the \emph{optimal robust value} $V^{*}_{\tilde{M}}(s)$, which is the best possible value under the worst-case environment policy. Formally:
\begin{align}
V^{*}_{\tilde{M}}(s) &= \sup_{\pi} \inf_{\tau} V^{\pi,\tau}_{\tilde{M}}(s), \text{ and}\\
\pi^* &= \argsup_{\pi} \inf_{\tau} V^{\pi,\tau}_{\tilde{M}}(s).
\end{align}
In this paper, we implicitly assume that the agent aims to maximise the objective while the environment adversarially seeks to minimise it. All results remain valid under the dual case with reversed roles~\citep{DBLP:journals/ior/NilimG05}. It is straightforward to verify that the policy $\pi^*$ guarantees at least the value $V^{*}_{\tilde{M}}(s)$ on any concrete f-MDP obtained by fixing specific distributions from the uncertainty sets.

Next, in Section \ref{sec:solving}, we present novel methods for efficiently and accurately solving rf-MDPs, i.e., computing optimal robust values and policies, assuming polytopic marginal uncertainty sets, such as the commonly used $L_1$, $L_\infty$ balls and general $L_p$ balls. Then, in Section \ref{sec:learning}, we leverage these methods to efficiently learn robust policies with provable performance guarantees in unknown f-MDPs.

\section{Solving Robust Factored MDPs}
\label{sec:solving}

As for standard robust MDPs, the optimal value function \(V^{*}_{\tilde{M}}\) and a corresponding robust policy for an rf-MDP can be computed with 
\emph{robust value iteration}~\citep{DBLP:journals/mor/Iyengar05,DBLP:journals/ior/NilimG05}.  
Assuming \emph{rectangular} uncertainty sets, meaning that each state–action pair has an independent uncertainty set over which the environment can act adversarially, the global problem decouples into a local optimisation at every state. For any state \(s\), the agent selects an action \(a\in A\) that maximises the worst-case expected return over all transition kernels in \(\mathcal{T}(s,a)\), yielding the robust Bellman equation, where \(V^{*}_{\tilde{M}}(s)\) equals:
\begin{equation}
\label{eq:bellman}
 \max_{a\in A}
  \underbrace{%
      \min_{T\in\mathcal T(s,a)}\!\bigl[r(s,a)\;+\gamma\sum_{s'\in S}T(s'|s,a)\,V^{*}_{\tilde{M}}(s')\bigr].
  }_{\text{Inner Optimisation}}
\end{equation}
The inner optimisation captures the environment’s adversarial choice of a transition kernel within \(\mathcal{T}(s,a)\).
For standard (non-factored) robust MDPs,
this is tractable when \(\mathcal{T}(s,a)\) has a favourable geometry: e.g., an \(L_{1}\) or \(L_{\infty}\) ball, which is solvable via bisection in time linear-logarithmic in the support size~\citep{DBLP:conf/icml/StrehlL05}, or a polytope described by a number of vertices or half-spaces that can be solved via linear programming~\citep{DBLP:journals/ior/NilimG05}.  

rf-MDPs, however, induce uncertainty sets \(\mathcal{T}(s,a)\) as the multilinear product of marginal sets (see Equation~\eqref{eq:transunc}). Even if every marginal \(\mathcal{P}(D_i(s,a))\) is convex, convexity is in general not preserved under the product;  consequently, \(\mathcal{T}(s,a)\) can in general be non-convex (see Figure~\ref{fig:product}), rendering the inner optimisation hard and often intractable.  

We show that when the marginals are polytopes, the associated non-linear problem admits an exact linear reformulation whose constraints follow directly from the polytopic descriptions of the marginals. However, the number of resulting constraints can grow rapidly for many common classes of uncertainty sets. To retain tractability, we construct tight linear overapproximations of \(\mathcal{T}(s,a)\), yielding robust Bellman updates that allow for an efficient and accurate solution.

\subsection{Exact Products of Polytopic Uncertainty Sets}

We consider polytopic marginal uncertainty sets $\mathcal{P}$ defined as the convex hull of finitely many extreme distributions, i.e.,
$
\mathcal{P} = \mathrm{conv}\{P^{(1)}, \dots, P^{(m)}\} := \{ \sum_{i=1}^m \lambda_i P^{(i)} \,|\, \lambda_i \ge 0,\ \sum_{i=1}^m \lambda_i = 1 \}.
$
We first prove that the resulting inner optimisation problem in~\eqref{eq:bellman}, taken over the non-convex product uncertainty set $\mathcal{T}(s,a)$, admits an exact linear reformulation. In contrast to prior approaches for solving robust factored MDPs~\citep{DBLP:journals/ai/DelgadoSB11}, this result allows us to avoid the invocation of an expensive and potentially approximate non-linear solver. It builds on two key observations: first, by the bilinearity of the Kronecker product $\otimes$~\cite{DBLP:books/cu/HornJ91}, the convex hull of $\mathcal{T}(s,a)$ is a polytope whose extreme points are precisely the pairwise products of the extreme distributions of the marginal polytopes~\cite{DBLP:books/sp/HorstT96}, and second, the inner optimisation is linear in the transition probabilities and thus attains its optimum at a vertex of the convex hull.

\setcounter{theorem}{0}
\begin{theorem}\label{thm:vertex}
  Let 
  $\mathcal{P} = \mathrm{conv}\{P^{(1)}, \dots, P^{(m)}\} \subseteq \Delta_M$ 
  and 
  $\mathcal{Q} = \mathrm{conv}\{Q^{(1)}, \dots, Q^{(k)}\} \subseteq \Delta_N$
  be polytopic marginal uncertainty sets. Then the corresponding non-linear inner optimisation problem in Equation~\eqref{eq:bellman} attains its optimum at one of the products of the marginal extreme distributions:
  \[
    \left\{\, P^{(i)} \otimes Q^{(j)} \;\middle|\; 1 \le i \le m,\; 1 \le j \le k \,\right\}.
  \]
\end{theorem}
\ifthenelse{\isundefined{\techreport}}{The full proofs for all presented results are provided in the extended version.}{The full proof is provided in Appendix~\ref{app:proofs}.}
Theorem~\ref{thm:vertex} inductively extends to any number of marginals and offers a direct approach to solving the inner optimisation problem \emph{exactly} by enumerating the product vertices induced by the marginal uncertainty sets of each factor. However, the number of such vertices can grow rapidly, rendering explicit enumeration computationally infeasible, even for standard classes of uncertainty sets arising from statistical estimation~\citep{DBLP:conf/birthday/SuilenBB0025}. For example, when the marginal sets are defined as $L_1$ or $L_\infty$ balls centred around a nominal distribution, the number of vertices per marginal can grow exponentially in the support size. \ifthenelse{\isundefined{\techreport}}{A detailed construction can be found in the extended version.}{A detailed construction can be found in Appendix~\ref{app:vertexgrwoth}.}

\subsection{Efficient Solutions through Relaxations}
\label{sec:relaxations}

To mitigate the potential intractability of the exact inner optimisation, we use \emph{relaxations}, i.e., overapproximations of the uncertainty set $\mathcal{T}(s,a)$ that trade exactness for tractability. Since the relaxed set is a superset of the true one, the value returned by the relaxed Bellman operator is a \emph{lower bound} on the exact robust value, guaranteeing that the resulting policy never underperforms against any transition kernel in the original uncertainty set. This sound, worst-case guarantee distinguishes our approach from earlier methods for rf-MDPs, which rely on approximate value‐function fitting over a fixed basis~\citep{conf/isipta/DelgadoBCS09,DBLP:journals/ai/DelgadoSB11,journals/corr/abs-2404-02006}. Such schemes provide no formal bound on the policy’s performance and therefore cannot in general provide safety guarantees as required in robust learning.  

\begin{figure}[t]
  \centering
  \includegraphics[width=0.472\textwidth]{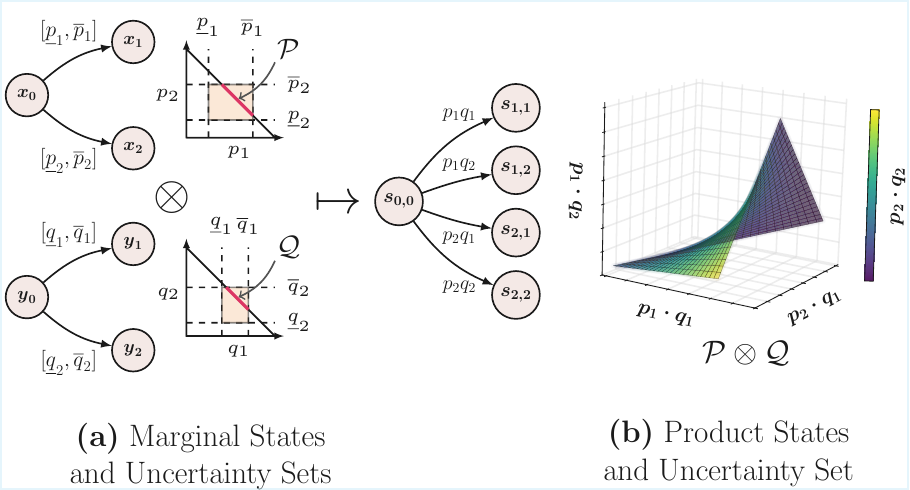}
  \caption{Part (a) shows two factors of an rf-MDP, with convex marginal uncertainty sets \(\mathcal{P}\) and \(\mathcal{Q}\), which are line segments in the two-dimensional probability simplex. The resulting product uncertainty set $\mathcal{P} \otimes \mathcal{Q}$ in (b) is non-convex.}
  \label{fig:product}
\end{figure}

We aim for  \emph{tight} relaxations, admitting as few spurious distributions (i.e., distributions not in the true set) as possible. An overly loose relaxation can lead to a pessimistic bound, and result in an unnecessarily conservative policy.  

We first consider marginal uncertainty sets that take the form of \emph{boxes} (or \emph{hyper-rectangles}) intersected with the probability simplex, which are generalisations of $L_\infty$ balls. 

These arise naturally when individual transition probabilities are estimated from observed data using confidence intervals~\citep{DBLP:conf/icml/StrehlL05,DBLP:conf/nips/SuilenS0022}.
A box is defined by lower and upper bounds $\underline{p}, \overline{p} \in [0,1]^N$ on each component of a probability distribution, with $\underline{p}_i \leq \overline{p}_i$ for all $1 \leq i \leq N$, yielding the uncertainty set
\begin{equation}
    \mathcal{P}_B = \left\{ (p_1, \dots, p_N) \in \Delta_N \,\middle|\, \underline{p}_i \leq p_i \leq \overline{p}_i \right\}.
\end{equation}
Robust MDPs defined in this way are called \emph{interval} or \emph{bounded-parameter MDPs}~\citep{DBLP:journals/ai/GivanLD00}.

\paragraph{Interval-Arithmetic Relaxation.}
A natural relaxation for products of distributions in interval MDPs is to use \emph{interval arithmetic}. In fact, this approach is taken in the modelling language of the PRISM tool~\citep{DBLP:conf/cav/KwiatkowskaNP11}, which supports compositional modelling of interval MDPs.
Given two box-type uncertainty sets $\mathcal{P}_B \subseteq \Delta_M$ and $\mathcal{Q}_B \subseteq \Delta_N$ with respective bounds $\underline{p}, \overline{p} \in [0,1]^M$ and $\underline{q}, \overline{q} \in [0,1]^N$, the corresponding interval-arithmetic relaxation $\mathcal{R}_{ia} \subseteq \Delta_{M\cdot N}$ is defined as
\begin{equation}
    \mathcal{R}_{ia} = \left\{ H \in \Delta_{M\cdot N} \,\middle|\, \underline{p}_i \, \underline{q}_j \leq h_{ij} \leq \overline{p}_i \, \overline{q}_j \right\}.
\end{equation}
While the interval-arithmetic relaxation is tight with respect to each component individually, it fails to capture dependencies across components and can therefore introduce a large amount of spurious 
distributions~\citep{DBLP:journals/corr/HashemiHT16,DBLP:conf/hybrid/MathiesenHL25}. In particular, it admits  spurious extreme points, potentially leading to overly conservative solutions in the inner optimisation problem, as we demonstrate in the following example.

\begin{example}
Consider the two box-type uncertainty sets:
\begin{align*}
&\mathcal{P}_B=\{(p,1-p) \in \Delta_2 \mid p\in[0.2,0.6]\}, \text{and}\\
&\mathcal{Q}_B=\{(q,1-q) \in \Delta_2 \mid q\in[0.1,0.3]\}.
\end{align*}
Their interval-arithmetic product relaxation $\mathcal{R}_{ia}$ is:
\[
[0.02,0.18]\times[0.14,0.54]\times[0.04,0.24]\times[0.28,0.72] \cap \Delta_4.\]
Now consider 
\(
H=(0.18,\,0.14,\,0.24,\,0.44)\in\mathcal{R}_{ia}.
\)
We can verify that $H$ is a vertex of $\mathcal{R}_{ia}$, as three bounds are tight. Since \(h_1=0.18=pq\), the box constraints imply that \(p=0.6\) and \(q=0.3\). But then it must be that:
\[
\big(p(1-q),\,(1-p)q,\,(1-p)(1-q)\big)
=(0.42,\,0.12,\,0.28),
\]
so the only valid product distribution with \(pq=0.18\) is \((0.18,0.42,0.12,0.28)\neq H\). Thus \(H\)  is not contained in the actual product uncertainty set \(\mathcal{P}_B \otimes \mathcal{Q}_B\). 
\end{example}

\paragraph{McCormick Relaxation.}

In order to tackle the issue of spurious distributions in interval-arithmetic relaxations,
and the conservative solutions to the inner optimisation problem in~\eqref{eq:bellman} that may result,
we draw on results from non-linear global optimisation and employ \emph{McCormick envelopes}~\citep{DBLP:journals/mp/McCormick76}. These provide tight convex relaxations of multilinear products through a polynomial number of linear constraints, yielding a tractable linear program that closely approximates the original non-linear formulation.

For two variables \(p\in[\underline{p},\overline{p}]\) and \(q\in[\underline{q},\overline{q}]\), the McCormick envelopes 
are defined by the following linear inequalities:
\begin{subequations}
\label{eq:mccormick}
\begin{align}
h &\ge p\,\underline{q} + q\,\underline{p} - \underline{p}\,\underline{q}, \label{eq:mc-lb1}\\
h &\ge p\,\overline{q} + q\,\overline{p} - \overline{p}\,\overline{q}, \label{eq:mc-lb2}\\
h &\le p\,\underline{q} + q\,\overline{p} - \overline{p}\,\underline{q}, \label{eq:mc-ub1}\\
h &\le p\,\overline{q} + q\,\underline{p} - \underline{p}\,\overline{q}. \label{eq:mc-ub2}
\end{align}
\end{subequations}
Each inequality arises from combining the bounds on \(p\) and \(q\). For instance, since \(p \ge \underline{p}\) and \(q \ge \underline{q}\), we have
\[
(p - \underline{p})(q - \underline{q}) \ge 0.
\]
Expanding and substituting \(h = pq\) gives
\[
pq - p\,\underline{q} - q\,\underline{p} + \underline{p}\,\underline{q} \ge 0
\,\Longrightarrow\,
h \ge p\,\underline{q} + q\,\underline{p}  - \underline{p}\,\underline{q},
\]
which is precisely Equation~\eqref{eq:mc-lb1}. Despite their simplicity, 
these inequalities suffice to exactly characterise the convex hull of a single bilinear product $h = pq$~\citep{DBLP:journals/mp/McCormick76}. 

When applied to the inner optimisation in Equation~\eqref{eq:bellman} over a product uncertainty set as per Equation~\eqref{eq:transunc}, each bilinear term \(p_i q_j\) is replaced by an auxiliary variable \(h_{ij}\), which is constrained by the four McCormick inequalities in~\eqref{eq:mccormick}. We then impose the global simplex constraint \(\sum_{i,j} h_{ij} = 1\),
ensuring that the auxiliaries \(\{h_{ij}\}_{i,j}\) define a valid probability distribution. This reformulation linearises the original non-linear inner optimisation. 
Figure~\ref{fig:iaVSmc} illustrates how the McCormick relaxation excludes many of the spurious extreme points admitted by the interval-arithmetic relaxation, thus resulting in less conservative solutions and more effective (whilst still robust)  policies. Furthermore, since each \(h_{ij}\) contributes to exactly four McCormick constraints, 
the total number of constraints grows only polynomially with the marginal supports, yielding a tractable inner linear program.
Full details of this construction and its extension to products of more than two marginal uncertainty sets (obtained by recursive applications) are provided in 
\ifthenelse{\isundefined{\techreport}}{the extended version.}{Appendix~\ref{app:mccormick}.}

\begin{figure}
    \centering
    \includegraphics[width=0.78\linewidth]{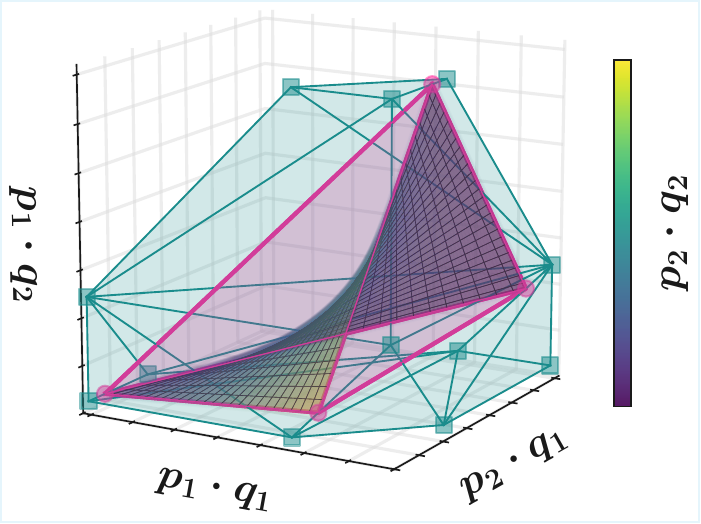}
\caption{Projections of the interval-arithmetic (blue) and McCormick (pink) relaxations for the product of box-type uncertainty sets (coloured curve). The McCormick relaxation is tighter and has fewer spurious extreme distributions.}
    \label{fig:iaVSmc}
\end{figure}

\paragraph{Relaxations for $L_p$ Uncertainty Sets.} 
The constructions above enable the exact composition of polytopic uncertainty sets and provide tight-yet-tractable relaxations for box-type uncertainty sets. We now also consider uncertainty sets that are \(L_p\) norm balls centred at a nominal distribution \(\hat{P} \in \Delta_N\), which are typically estimated from observed data as:
\begin{equation*}
    \mathcal{P}_p(\hat{P}, \varepsilon) = \{P \in \Delta_N \mid \|P - \hat{P}\|_p \leq \varepsilon\}.
\end{equation*}
These sets are generally not polytopic, for $1 < p < \infty$. We hence extend a result from~\citet{Strehl2007}, originally formulated for the composition of \(L_1\) balls, to arbitrary \(L_p\) norms:

\setcounter{theorem}{1}
\begin{theorem}
  \label{thm:radiisum}
  Let \(\mathcal{P}_p(\hat{P}, \varepsilon_1)\) and \(\mathcal{P}_p(\hat{Q}, \varepsilon_2)\) be two \(L_p\) uncertainty sets for some \(1 \le p \le \infty\). Then:
  \[
    \mathcal{P}_p(\hat{P}, \varepsilon_1)\otimes\mathcal{P}_p(\hat{Q}, \varepsilon_2)
    \subseteq \mathcal{P}_p(\hat{P}\otimes\hat{Q}, \varepsilon_1 + \varepsilon_2).
  \]
\end{theorem}

\ifthenelse{\isundefined{\techreport}}{}{We provide the proof in Appendix~\ref{app:proofs}. }
This result offers an approach to solving non-polytopic rf-MDPs, complementing the constructions presented in  Section~\ref{sec:relaxations}. When applied to $L_1$ uncertainty sets, it directly extends the PAC analysis of~\citet{Strehl2007} to robust policy synthesis. In Section~\ref{sec:experiments}, we compare the various relaxations, showing that our constructions yield substantially tighter uncertainty sets, enable more sample-efficient learning, and deliver robust policies with stronger performance guarantees.

\section{Robust Policy Learning in Factored MDPs}
\label{sec:learning}

We now introduce a novel learning approach that integrates factored model estimation with accurate and tractable robust planning, generating policies that are provably robust for unknown f-MDPs. Based on agent interactions with the environment, we derive marginal uncertainty sets, such as confidence intervals or $L_1$ balls, which induce a polytopic rf-MDP. Leveraging the solution methods in the previous section, we exploit this factored structure to achieve dimensional gains in sample efficiency compared to existing robust learning methods in flat models, as we demonstrate in our experimental evaluation. Crucially, our approach provides a finite-sample, anytime PAC guarantee: after any number of interactions, we can bound the worst-case performance in the unknown MDP with high confidence.

We consider a factored MDP \(M\) 
with known state space but 
unknown (marginal) transition distributions.  For clarity, we assume that the reward function is known, but all results extend to the case of unknown reward functions~\citep{DBLP:conf/icml/StrehlL05}.  Our algorithm has access to agent-environment interactions in the form of a dataset of transition samples \(\mathcal{C} = \{(s_t, a_t, s_t')\}_t\), where \(a_t\) is the action taken in state \(s_t\) under some exploration policy and \(s'_t\) is the observed successor state.  
We remain agnostic to the precise sampling mechanism and assume that the sample set \(\mathcal{C}\) is given. In Section~\ref{sec:experiments}, we describe the specific sampling procedure used in our evaluation.

From the definition of a factored MDP, we first identify the relevant transition components that must be estimated.  For a state–action pair \((s,a)\), the relevant dependencies are
\begin{equation*}
    D_{s,a}
=\bigl\{j\in \mathcal{I}\;\big|\;\exists i. \; j=D_i(s,a)\bigr\}.
\end{equation*}
Aggregating over all state–action pairs yields the set of relevant transition components:
\(
    \mathcal{Q} \;=\; \bigcup_{(s,a)\in S\times A} D_{s,a},
\)
so that \(\lvert\mathcal{Q}\rvert\) counts the number of marginal transition distributions to be estimated.  The total number of unknown transition probabilities is the sum of the supports of the marginals:
\(
    U \;=\; \sum_{j\in\mathcal{Q}}
\bigl\lvert\mathrm{supp}\bigl(P(\,\cdot \,|\, j)\bigr)\bigr\rvert.
\)
For a sample dataset \(\mathcal{C} = \{(s_t, a_t, s_t')\}_t\), we define the \emph{realisation counts}:
\[
n(x_i, j)
= \sum_{(s,a,s') \in \mathcal{C}}
  \mathds{1}\bigl(D_i(s,a)=j \;\land\; s_i'=x_i\bigr), 
\]
and the \emph{component counts}:
\[
n(j)
= \sum_{(s,a,s') \in \mathcal{C}}
  \sum_{i}
  \mathds{1}\bigl(D_i(s,a)=j\bigr),
\]
for \(x_i\in\mathcal{D}_i\) and \(j\in\mathcal{I}\).
Here, \(n(j)\) is the total number of encountered transitions whose transition probability distribution involves a marginal with dependency identifier \(j\), while \(n(x_i,j)\) records how often such transitions lead to the marginal state component \(x_i\). From this we can derive the empirical estimates of the marginal distributions as
\begin{equation}
    \label{eq:pointestimate}
    \hat{P}(s_i' | D_i(s,a)) = \frac{n(s_i' , D_i(s,a))}{n(D_i(s,a))}.
\end{equation}
While this empirical estimate becomes increasingly accurate with more data, it provides no quantification of uncertainty.  We aim to synthesise a policy that, after any fixed number of samples, comes with a guaranteed lower bound on its performance in the unknown f-MDP.  To achieve this, we inflate each point estimate into a high‐confidence uncertainty set over the marginal distribution, thereby defining an rf‐MDP.  

\subsection{Uncertainty Set Construction}
We consider two established methods for constructing uncertainty sets.  The first builds exact binomial confidence intervals for each transition probability, treating each outcome \(s_i'\) under dependency \(j = D_i(s,a)\) as a Bernoulli trial~\citep{DBLP:conf/nips/SuilenS0022,DBLP:conf/qestformats/MeggendorferWW25}.  Given \(x = n(s_i',j)\) “successes” in \(n = n(j)\) trials and an error probability \(\delta \in (0,1)\), the true transition probability \(P(s_i'\mid j)\) lies in the interval:
\begin{equation*}
\text{CP}(s_i',j) = \bigl[B\bigl(\tfrac\delta2; x,\,n-x+1\bigr),\;B\bigl(1-\tfrac\delta2; x+1,\,n-x\bigr)\bigr]
\end{equation*}
with probability at least \(1-\delta\), where \(B(\alpha;u,v)\) denotes the \(\alpha\)\nobreakdash-quantile of the \(\mathrm{Beta}(u,v)\) distribution~\citep{ClopperPearson1934}.  Applying these bounds independently to each transition component defines the box‐type uncertainty sets
\[
\mathcal{P}(j)
=\Bigl\{\,P'\in\Delta(\mathcal{D}_i)\;\Big|\;
P'(s_i')\in\text{CP}(s_i',j)\;\forall s_i'\Bigr\},
\]
to which our rf‐MDP solution techniques apply directly. Throughout, we assume $n(j) > 0$. When $n(j) = 0$, we set the uncertainty sets as the entire probability simplex.

The second approach centres on an \(L_{1}\)‐norm ball around the empirical marginal distribution \(\hat{P}(\cdot\,|\, j)\).  For each relevant dependency identifier \(j = D_i(s,a) \in \mathcal{Q}\) , we set
\[
\mathcal{P}(j)
=\Bigl\{\,P'\in\Delta(\mathcal{D}_i)\;\Big|\;
\|P'(\cdot) - \hat{P}(\,\cdot\,|\, j)\|_{1} \le \varepsilon\Bigr\},
\]
where $\varepsilon$ follows from~\citet{WeissmanEtAl2003} as
\[
\varepsilon = \sqrt{\frac{2\bigl[\ln(2^{a}-2)-\ln(\delta)\bigr]}{n(j)}},
\quad
a = \bigl\lvert\mathrm{supp}(P(\cdot\mid j))\bigr\rvert.
\]
This ensures that the true marginal lies in \(\mathcal{P}(j)\) with probability at least \(1 - \delta\). This underpins the native PAC-learning results for both factored and standard MDPs~\citep{DBLP:conf/icml/StrehlL05,Strehl2007}.
Moreover, it yields polytopic uncertainty sets, as the intersection of an $L_1$ ball with the probability simplex is a polytope, thus permitting exact composition via Theorem~\ref{thm:vertex}.
However, \(L_{1}\) balls do not integrate naturally into the McCormick relaxation without further overapproximating them as boxes. As we show in
\ifthenelse{\isundefined{\techreport}}{the extended version}{Appendix~\ref{app:l1box} }, overapproximating \(L_1\) balls by their smallest enclosing box yields a looser uncertainty set than applying the box-type construction directly. Consequently, the radius‐sum result of Theorem~\ref{thm:radiisum} is the natural choice when composing $L_1$ marginal sets with a large number of vertices. 

\begin{table*}[!htbp]
\centering
\resizebox{.99\linewidth}{!}{%
  \begin{tabular}{c c c | c c | c c c | c c c}
    \toprule
    \multirow{2}{*}{\textbf{Domain}} & \multirow{2}{*}{$\bm{|S|}$} & \multirow{2}{*}{$\bm{|T|}$}
      & \multicolumn{2}{c}{\makecell{\textbf{Vertex Enumeration}}}
      & \multicolumn{3}{c}{\makecell{\textbf{Interval-Arithmetic}}}
      & \multicolumn{3}{c}{\makecell{\textbf{McCormick}}} \\
    \cmidrule(lr){4-5} \cmidrule(lr){6-8} \cmidrule(lr){9-11}
    & & 
      & \makecell{\textbf{Robust Value}}
      & \makecell{\textbf{Time [s]}}
      & \makecell{\textbf{Robust Value}}
      & \makecell{\textbf{Rel. Gap}}
      & \makecell{\textbf{Time [s]}}
      & \makecell{\textbf{Robust Value}}
      & \makecell{\textbf{Rel. Gap}}
      & \makecell{\textbf{Time [s]}} \\
    \midrule
    \text{Aircraft ($\uparrow$)} & $11153$ & $1262099$
      & $0.73$ & $2535.8$ 
      & $0.65$ & $11\%$ & $6.1$
      & $0.73$ & $0\%$ & $43.7$ \\
    \text{Drone ($\uparrow$)} & $262144$ & $21694720$
      & $0.69$ & $2125.8$
      & $0.63$ & $10\%$ & $90.2$
      & $0.69$ & $0\%$ & $190.7$ \\
    \text{Stock Trading ($\uparrow$)} & $12481$ & $5362624$
      & $25.43$ & $67.6$
      & $17.60$ & $31\%$ & $16.0$
      & $25.43$ & $0\%$ & $67.5$ \\
    \text{SysAdmin ($\uparrow$)} & $15873$ & $9332587$
      & $50.70$ & $66.7$
      & $46.66$ & $8\%$ & $34.1$
      & $50.70$ & $0\%$ & $64.1$ \\
    \text{Chain ($\downarrow$)} & $100$ & $3136$
      & $331.34$ & $778.1$
      & $451.28$ & $36\%$ & $0.6$
      & $331.34$ & $0\%$ & $7.6$ \\
    \text{Frozen Lake ($\downarrow$)} & $50625$ & $1866556$
      & $216.01$ & $1018.4$
      & $242.05$ & $12\%$ & $67.7$
      & $216.01$ & $0\%$ & $105.9$ \\
    \text{Herman ($\downarrow$)} & $2048$ & $177148$
      & $20.64$ & $11.0$
      & $23.82$ & $15\%$ & $2.8$
      & $20.64$ & $0\%$ & $8.1$ \\
    \bottomrule
  \end{tabular}%
}
\caption{Results for solving rf-MDPs. Arrows ($\uparrow$/$\downarrow$) indicate optimisation directions. $|S|$ and $|T|$ denote the number of states and transitions. The relative gap is $|V_{\mathrm{VE}} - V_{R}|/V_{R}$, where $V_{\mathrm{VE}}$ and $V_{R}$ are the robust results from vertex enumeration and respective relaxation. The complete set of experiments, with more results for varying uncertainty radii, are in Table~\ref{tab:case_study_extended} of Appendix~\ref{app:experiments}.}
\label{tab:case_study}
\end{table*}

\captionsetup[sub]{justification=centering}
\begin{figure*}[!h]
  \centering
  \begin{subfigure}[b]{\textwidth}
    \centering
    \includegraphics[width=0.8\textwidth]{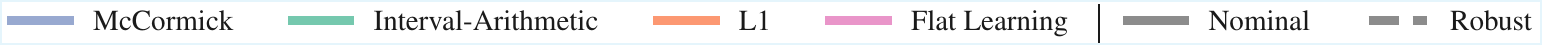 }
  \end{subfigure}
  \\ 

  \begin{subfigure}[b]{0.25\textwidth}
    \centering
    \includegraphics[width=\textwidth]{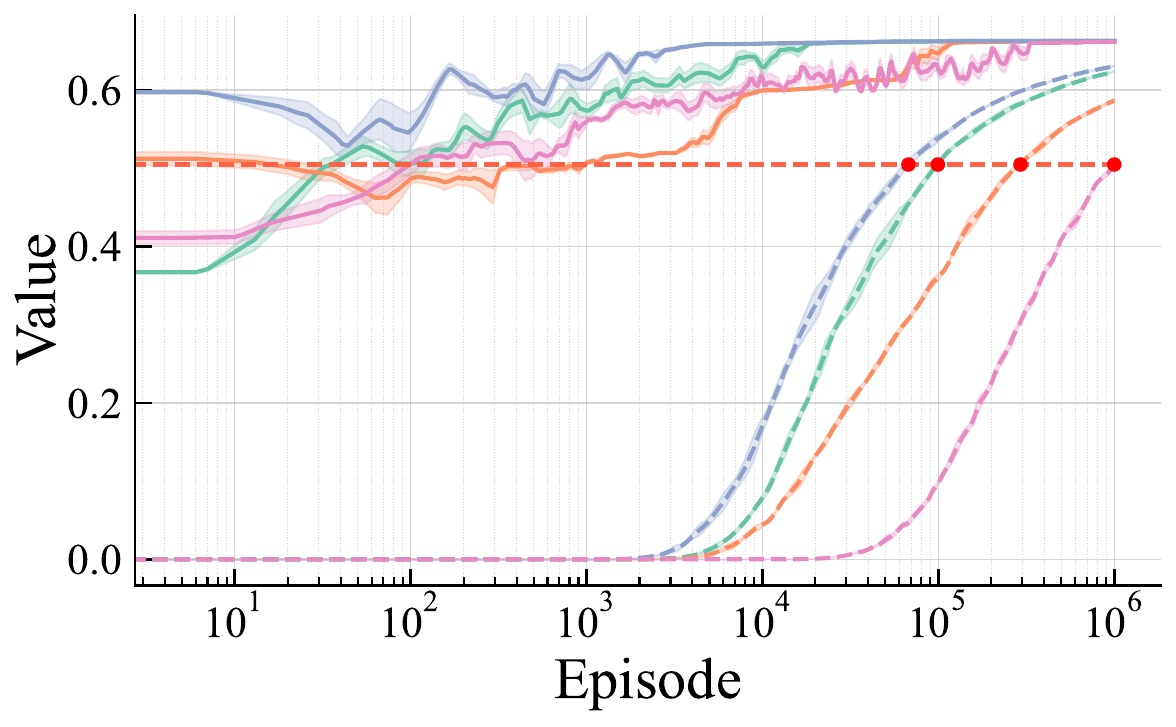}
    \caption{Aircraft}
    \label{fig:sub1}
  \end{subfigure}\hfill
  \begin{subfigure}[b]{0.25\textwidth}
    \centering
    \includegraphics[width=\textwidth]{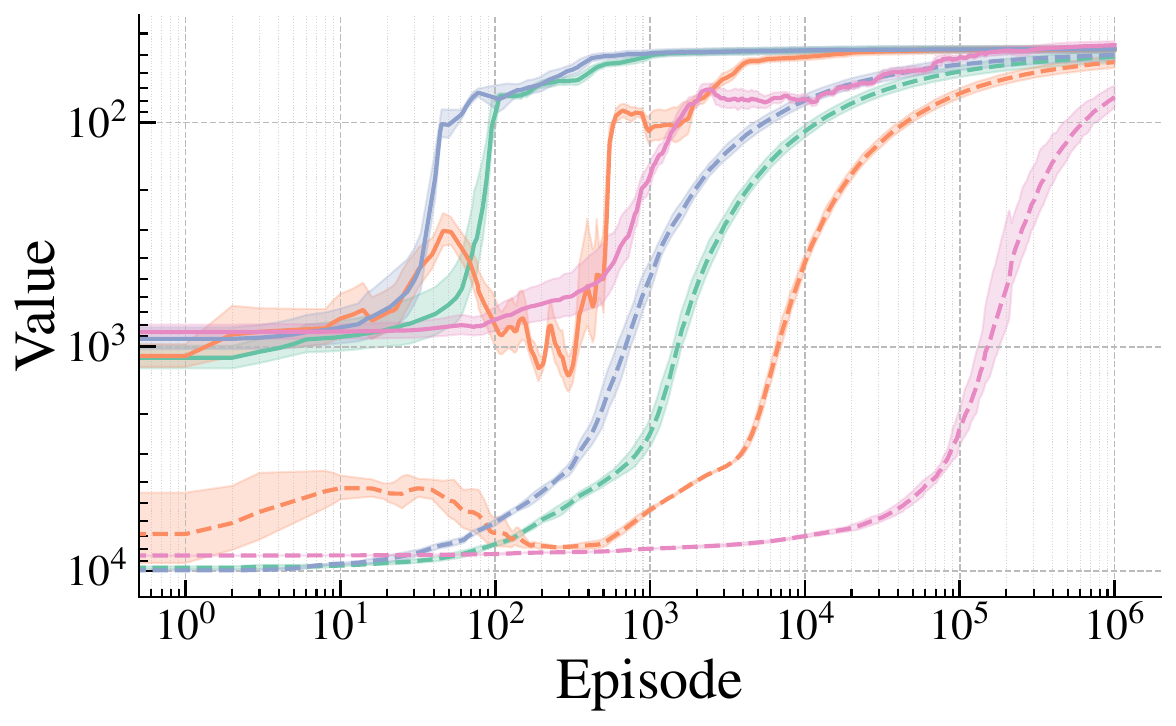}
    \caption{Frozen Lake}
    \label{fig:sub2}
  \end{subfigure}\hfill
  \begin{subfigure}[b]{0.25\textwidth}
    \centering
    \includegraphics[width=\textwidth]{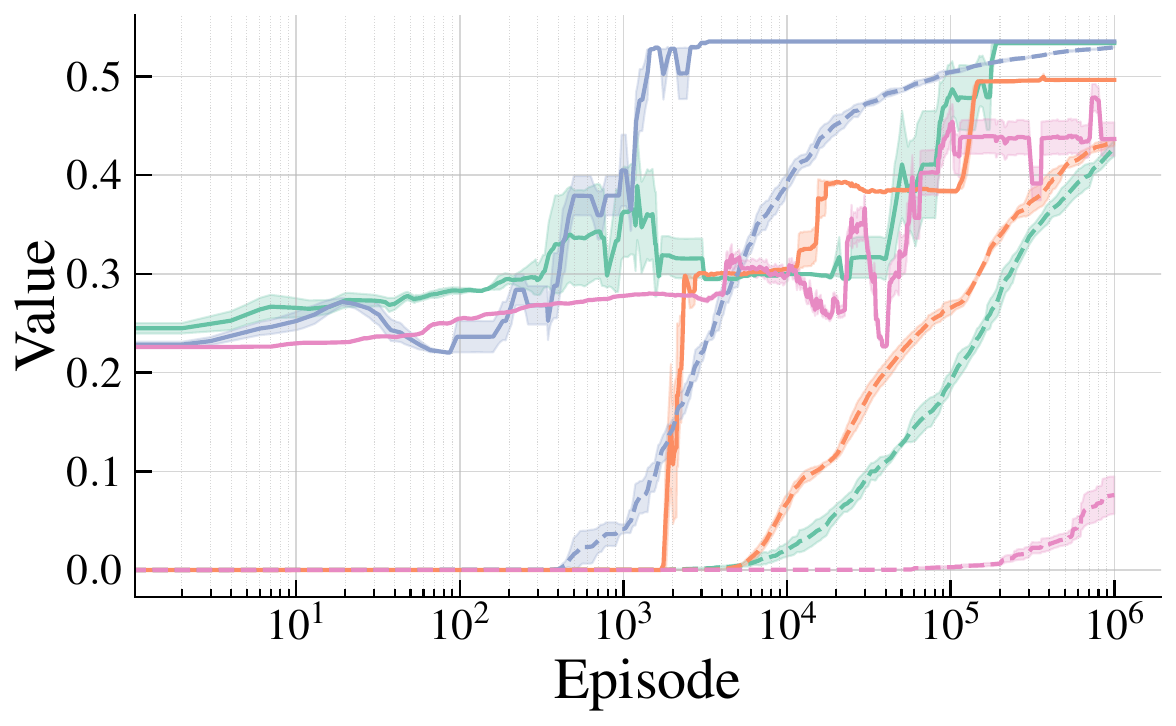}
    \caption{Drone}
    \label{fig:sub3}
  \end{subfigure}\hfill
  \begin{subfigure}[b]{0.25\textwidth}
    \centering
    \includegraphics[width=\textwidth]{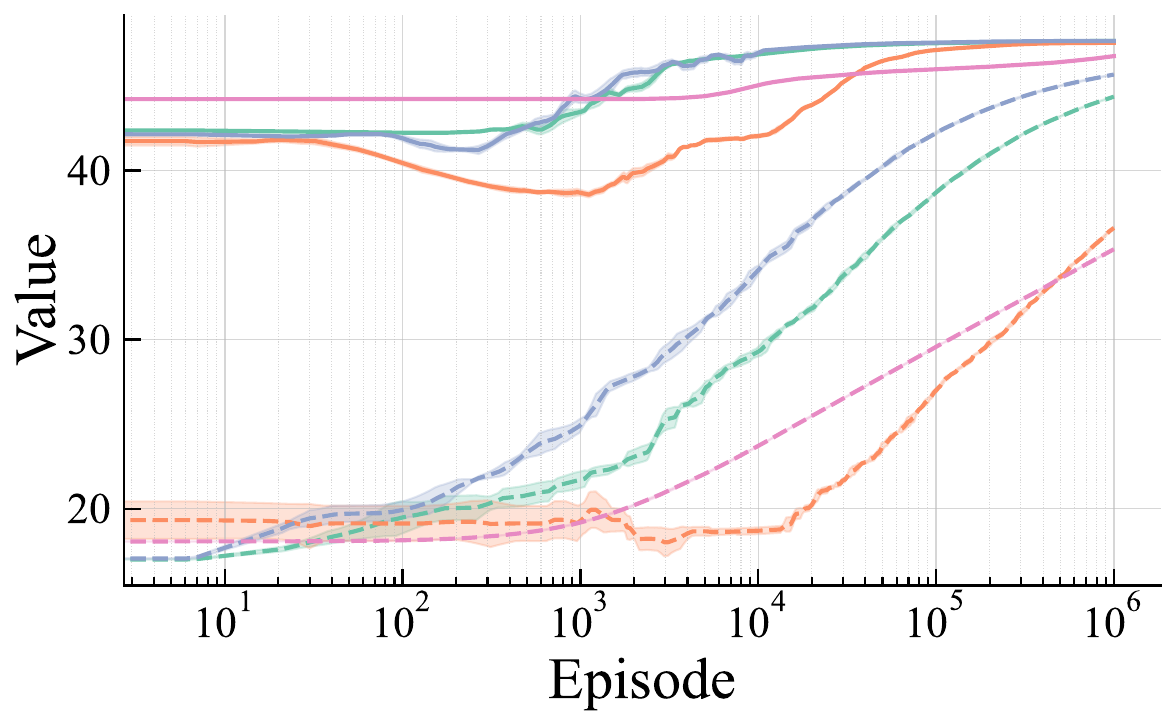}
    \caption{SysAdmin}
    \label{fig:sub4}
  \end{subfigure}

\caption{Results for robust policy learning.
The plots show objective value against processed fixed-length trajectories. Dashed curves show the robust guarantee for the learned robust policy, solid curves show its actual performance on the true model. The complete experimental results, including additional domains and total runtimes, are provided in Figure~\ref{fig:learning_extended} of Appendix~\ref{app:experiments}.}
  \label{fig:learningresults}
  \vspace*{-1em}
\end{figure*}

\subsection{Provably Robust Policy Synthesis}
To obtain a provably robust policy with quantifiable performance guarantees in the unknown f-MDP \(M\), we construct an rf-MDP \(\tilde{M}\) using the uncertainty sets described above. For the guarantees to be meaningful, we must ensure that the unknown MDP \(M\) is \emph{contained} in \(\tilde{M}\) (denoted \(M \in \tilde{M}\)) with high, user-specified confidence. This means that every marginal distribution \(P(\,\cdot \, | j)\) for \(j \in \mathcal{Q}\) must lie within its corresponding uncertainty set \(\mathcal{P}(j)\).

Given a desired overall confidence probability \(1 - \beta\), we follow the standard approach of~\citet{Strehl2007} and distribute the total error probability \(\beta \in (0,1)\) across all learnt distributions/transitions. Under the \(L_\infty\) scheme, this results in \(\delta = \beta / U\), and under the \(L_1\) scheme, in \(\delta = \beta / \lvert \mathcal{Q} \rvert\). By the union bound, this ensures that \(M \in \tilde{M}\) with probability at least \(1 - \beta\), regardless of the number of observed samples.

When solving the learned rf-MDP $\tilde{M}$ using a \emph{robust}, i.e., either exact or relaxation-based method from Section~\ref{sec:solving}, the following performance guarantee for the resulting robust policy on the true, unknown f-MDP $M$ follows immediately:

\begin{theorem}
    Let $M$ be an f-MDP and $\tilde{M}$ an rf-MDP such that $\Pr[M \in \tilde{M}] \geq 1 - \beta$ for some $\beta > 0$. Let $\pi^*$ be the policy obtained by solving $\tilde{M}$ with a robust solution method, and let $V^{\pi^*}_{\tilde{M}}(s)$ denote its corresponding robust value. Then,
    \begin{equation}
    \label{eq:pacguarantee}
    \Pr\Bigl[V^{\pi^*}_M(s) \;\ge\; V^{\pi^*}_{\tilde{M}}(s)\Bigr] \;\ge\; 1 - \beta.
    \end{equation}
\end{theorem}

In other words, with probability at least \(1 - \beta\), the learned robust policy \(\pi^*\) achieves a value in every state of the true f-MDP that is no worse than its computed value in the learned rf-MDP. This PAC-style guarantee based on the novel robust solution methods distinguishes our approach from prior methods~\citep{DBLP:journals/ai/DelgadoSB11,journals/corr/abs-2404-02006}, which cannot guarantee a valid lower bound, thus forfeiting such a performance guarantee.

\section{Experiments}
\label{sec:experiments}

We integrated our methods into the PRISM solver for probabilistic models~\cite{DBLP:conf/cav/KwiatkowskaNP11}, which offers a modular language for specifying factored MDPs.  We augment PRISM with our algorithms for solving and learning robust factored MDPs and employ the Gurobi optimiser with default parameters for all linear programs.

\subsection{Evaluation: Solving rf-MDPs}

We evaluate the three methods for \emph{solving} rf-MDPs with box-type uncertainty sets: vertex enumeration, interval-arithmetic relaxations, and McCormick relaxations, across a range of benchmark environments.  These include classic f-MDP domains such as the System Administrator domain discussed in Example \ref{ex:1}~\cite{DBLP:conf/icml/GuestrinPS02}, 
as well as established r-MDP case studies with inherent factored structure, including multi-agent scenarios like the Aircraft Collision Avoidance domain~\cite{kochenderfer2015,DBLP:conf/tacas/SchnitzerAP25}.  Detailed descriptions of each domain
are provided in \ifthenelse{\isundefined{\techreport}}{the extended version.}{Appendix~\ref{app:benchmarks}.}
For each domain, we obtain an rf-MDP by perturbing a nominal transition kernel with an $L_{\infty}$ uncertainty radius of $0.025$ (see 
\ifthenelse{\isundefined{\techreport}}{the extended version}{Appendix~\ref{app:experiments}}
for additional levels of uncertainty), yielding box‐type uncertainty sets for each factor.

\subsubsection{Results.}
Table~\ref{tab:case_study} summarises the outcomes.
For each method, we report: (i) the robust value of the optimal policy in the rf-MDP; 
(ii) the runtime to solve the rf-MDP; and (iii) for relaxation-based methods, the relative gap to the exact result obtained by vertex enumeration, quantifying the additional conservatism introduced by 
over-approximating the product uncertainty sets.

Notably, McCormick relaxations preserve the tightness of vertex enumeration while remaining computationally efficient. Interval-arithmetic relaxations, though generally fast, yield looser bounds due to spurious extreme distributions. Overall, McCormick relaxations strike the best balance between solution tightness and runtime. We present the complete set of experiments, including analyses across varying uncertainty radii in \ifthenelse{\isundefined{\techreport}}{the extended version}{Appendix~\ref{app:experiments}}. 

\subsection{Evaluation: Robust Policy Learning in f-MDP}

We next compare four methods for robust policy \emph{learning}: (i) standard r-MDP learning in the flat model with box-type uncertainty sets; (ii) rf-MDP learning with $L_1$ uncertainty sets solved using the radius-sum result from Theorem~\ref{thm:radiisum}, which is the direct extension of the PAC analysis of~\citet{Strehl2007} to robust policy learning and represents the only available baseline for rf-MDPs; (iii) \& (iv) rf-MDP learning with box-type marginal uncertainty sets solved via either interval-arithmetic or McCormick relaxation.

To build the transition dataset $\mathcal{C}$, we iteratively sample fixed-length trajectories that restart in the initial state. 
To balance exploration and exploitation, we follow the \emph{optimism in the face of uncertainty} principle~\cite{DBLP:journals/ftml/Munos14}, selecting actions that are optimal under the most favourable transition model within the current uncertainty sets. 
Note that this choice of sampling procedure is arbitrary: the resulting robustness guarantees hold under any alternative sampling strategy, such as random action selection.

Across all domains, we fix the overall confidence level for the inclusion of the true, unknown MDP in the learned r-MDP to $1-\beta = 0.9999$, (see Equation~\eqref{eq:pacguarantee}). Each experiment is repeated with 10 distinct random seeds, and we report the average results along with standard deviation bands.

\subsubsection{Results.}
Figure~\ref{fig:learningresults} presents robust policy learning results across various domains. For each method, we plot the robust value of the learned  policy (dashed lines) and its nominal performance on the true, hidden model (solid lines) against the number of processed trajectories. While true-model performance provides useful validation, our focus lies on the robust values, i.e., the performance that can be guaranteed with high confidence on the unknown environment. 

The results demonstrate significant gains in sample efficiency by exploiting factored structures. Specifically, far fewer fixed-length trajectory samples are required to achieve equivalent robust performance guarantees compared to state-of-the-art methods on flat models. 
Furthermore, rf-MDP learning with box-type uncertainty sets, derived from exact confidence intervals and solved via convex relaxations, consistently outperforms approaches based on $L_1$ uncertainty sets and the radius-sum method. McCormick relaxations need about half the number of samples of interval-arithmetic relaxations for the same robust guarantees. This advantage is particularly crucial in domains where data collection is inherently limited, costly, or challenging.

Figure~\ref{fig:learningresults}a (red line) shows the number of samples needed to match the performance guarantee from flat learning on the Aircraft domain after $10^6$ trajectories. Factored learning with $L_1$ uncertainty sets reduces this to $3\cdot10^5$. Interval-arithmetic relaxation further decreases it to $10^5$, and McCormick relaxation is the most efficient, requiring only $6\cdot10^4$ trajectories. This gap becomes even more pronounced in other domains. 
We provide the full set of experiments including additional domains, total runtimes and detailed comparisons of sample efficiency in the Appendix~\ref{app:experiments}.

\section{Conclusion}

We have presented novel methods for solving robust factored MDPs, facilitating exact solutions and optimal robust policies for polytopic uncertainty sets. Utilising global optimisation techniques, we developed relaxation-based approaches that balance accuracy and computational tractability.
Our experimental results show that these methods markedly improve accuracy in solving rf-MDPs and enable significantly more sample-efficient robust policy learning.

\section*{Acknowledgements}
This work was partially supported by the ARIA projects \textsc{SAINT} and \textsc{Super Martingale Certificates}, the UKRI AI Hub on Mathematical Foundations of AI, and the European Research Council (ERC) under the European Union’s Horizon 2020 research and innovation programme (grant agreement No.~834115, \textsc{FUN2MODEL}). The authors are grateful to Karan Mukhi for the insightful discussions on this work.

\bibliography{aaai2026}

\clearpage
\appendix

\section{Proofs}
\label{app:proofs}
\subsection{Proof of Theorem 1}
\setcounter{theorem}{0}

\begin{theorem}[Restated]
  Let 
  $\mathcal{P} = \mathrm{conv}\{P^{(1)}, \dots, P^{(m)}\} \subseteq \Delta_M$ 
  and 
  $\mathcal{Q} = \mathrm{conv}\{Q^{(1)}, \dots, Q^{(k)}\} \subseteq \Delta_N$
  be polytopic marginal uncertainty sets. Then the corresponding inner optimisation problem in Equation~\eqref{eq:bellman} attains its optimum at one of the pairwise products of the extreme distributions:
  \[
    \left\{\, P^{(i)} \otimes Q^{(j)} \;\middle|\; 1 \le i \le m,\; 1 \le j \le k \,\right\}.
  \]
\end{theorem}
\begin{proof}

We define the set
\[
\mathcal{V} = \{P^{(i)} \otimes Q^{(j)} \mid 1 \leq i \leq m,\ 1 \leq j \leq k\},
\]
as the set of product vertices of the two polytopic marginal uncertainty sets, and let $\mathcal{S} = \mathcal{P} \otimes \mathcal{Q}$ denote their true product.
We first show that $\mathrm{conv}(\mathcal{S}) \subseteq \mathrm{conv}(\mathcal{V})$.

Let $P \in \mathcal{P}$ and $Q \in \mathcal{Q}$ be arbitrary. Since $\mathcal{P}$ and $\mathcal{Q}$ are convex hulls of their respective vertices, we can write:
\begin{align*}
    &P = \sum_{i=1}^{m} \lambda_i P^{(i)}, \quad \text{where } \lambda_i \ge 0,\ \sum_{i=1}^{m} \lambda_i = 1, \text{and}\\
&Q = \sum_{j=1}^{k} \theta_j Q^{(j)}, \quad \text{where } \theta_j \ge 0,\ \sum_{j=1}^{k} \theta_j = 1.
\end{align*}

By bilinearity of the Kronecker product $\otimes$~\citep{DBLP:books/sp/HorstT96,DBLP:conf/jmda/PadrolPfeifle10}, we have
\begin{align*}
P \otimes Q
  &= \Bigl( \sum_{i=1}^{m} \lambda_i P^{(i)} \Bigr)
     \otimes
     \Bigl( \sum_{j=1}^{k} \theta_j Q^{(j)} \Bigr) \\
  &= \sum_{i=1}^{m} \sum_{j=1}^{k}
     \lambda_i \theta_j \bigl( P^{(i)} \otimes Q^{(j)} \bigr).
\end{align*}
Thus, $P \otimes Q$ is a convex combination of elements in $\mathcal{V}$, so $P \otimes Q \in \mathrm{conv}(\mathcal{V})$. Since $P$ and $Q$ were arbitrary, we conclude:
\[
\mathcal{S} = \{P \otimes Q \mid P \in \mathcal{P},\ Q \in \mathcal{Q} \} \subseteq \mathrm{conv}(\mathcal{V}).
\]
 This implies that
\[
\mathrm{conv}(\mathcal{S}) \subseteq \mathrm{conv}(\mathrm{conv}(\mathcal{V})) = \mathrm{conv}(\mathcal{V}).
\]
On the other hand, it is clear that $\mathcal{V} \subseteq \mathcal{S}$, since each $P^{(i)} \otimes Q^{(j)}$ is the outer product of points from $\mathcal{P}$ and $\mathcal{Q}$, respectively. Therefore,
\(
\mathrm{conv}(\mathcal{V}) \subseteq \mathrm{conv}(\mathcal{S}).
\)
Putting both inclusions together yields the equality of the two convex hulls:
\[
\mathrm{conv}(\mathcal{S}) = \mathrm{conv}(\mathcal{V}).
\]
Thus, $\mathrm{conv}(\mathcal{V})$ is exactly the convex enclosure of $\mathcal{S}$.

Since the inner optimisation problem in Equation~\eqref{eq:bellman} is linear in the transition probabilities $T(s' \,|\, s,a)$, its optimum is attained at an extreme point of $\mathrm{conv}(\mathcal{S})$. The claim follows from the equality $\mathrm{conv}(\mathcal{S}) = \mathrm{conv}(\mathcal{V})$.
\end{proof}

\subsection{Proof of Theorem 2}

\setcounter{theorem}{1} 
\begin{theorem}[Restated]
Let \(\mathcal{P}_p(\hat{P}, \varepsilon_1)\) and \(\mathcal{P}_p(\hat{Q}, \varepsilon_2)\) be two \(L_p\) uncertainty sets for any \(1 \leq p \leq \infty\). Then,
\[
\mathcal{P}_p(\hat{P}, \varepsilon_1) \otimes \mathcal{P}_p(\hat{Q}, \varepsilon_2) \subseteq \mathcal{P}_p(\hat{P} \otimes \hat{Q}, \varepsilon_1 + \varepsilon_2). 
\] 
\end{theorem}

\begin{proof}
The proof follows from Minkowski's inequality~\citep{Rudin1987}, a generalisation of the triangle inequality, which establishes that for any two vectors \(u, v\in \mathbb{R}^n\):
\[
\|u + v\|_p \leq \|u\|_p + \|v\|_p.
\]
It suffices to show that for any distributions \(P, P' \in \Delta_M\) and \(Q, Q' \in \Delta_N\), it holds that
\[
\|P' \otimes Q' - P \otimes Q\|_p \leq \|P' - P\|_p + \|Q' - Q\|_p.
\]
By definition,
\begin{align*}
&P'(x)Q'(y) - P(x)Q(y)\\
= \quad &Q'(y)[P'(x) - P(x)] + P(x)[Q'(y) - Q(y)].
\end{align*}
Hence,
\begin{align*}
|P'(x)Q'(y) - P(x)Q(y)|^p =  &\big|Q'(y)[P'(x) - P(x)] \\
&+ P(x)[Q'(y) - Q(y)]\big|^p.
\end{align*}
Summing over \((x, y)\), taking the \(p\)-th root, and applying Minkowski's inequality gives:
\begin{align*}
&\|P' \otimes Q' - P \otimes Q\|_p\\
&\leq 
\left( \sum_{x, y} \big|Q'(y)[P'(x) - P(x)]\big|^p \right)^{1/p} \\
&\quad +
\left( \sum_{x, y} \big|P(x)[Q'(y) - Q(y)]\big|^p \right)^{1/p}.
\end{align*}
For the first term:
\begin{align*}
&\sum_{x, y} \big|Q'(y)[P'(x) - P(x)]\big|^p \\
=& \sum_y |Q'(y)|^p \sum_x |P'(x) - P(x)|^p \\
\leq& \sum_x |P'(x) - P(x)|^p,
\end{align*}
since \(\sum_y |Q'(y)|^p \leq \sum_y Q'(y) = 1\) for any \(p \geq 1\). Taking the \(p\)-th root yields
\[
\left( \sum_{x, y} \big|Q'(y)[P'(x) - P(x)]\big|^p \right)^{1/p} \leq \|P' - P\|_p.
\]
A symmetric argument shows
\[
\left( \sum_{x, y} \big|P(x)[Q'(y) - Q(y)]\big|^p \right)^{1/p} \leq \|Q' - Q\|_p.
\]
Combining both bounds concludes the proof:
\[
\|P' \otimes Q' - P \otimes Q\|_p \leq \|P' - P\|_p + \|Q' - Q\|_p.
\]
\end{proof}

\section{McCormick Linear Program}
\label{app:mccormick}

We present in detail the linear‐program (LP) construction that replaces the bilinear inner minimisation in Equation~\eqref{eq:bellman} with a convex relaxation based on McCormick envelopes, and show how this generalises to more than two marginal distributions by applying the construction recursively. 

\subsection{Linear Program Construction}
Recall that for two box‐type marginal uncertainty sets
\begin{align*}
\mathcal{P}_B &= \bigl\{P = (p_1,\dots,p_N)\in\Delta_N \mid \underline{p}_i \le p_i \le \overline{p}_i\bigr\},\\
\mathcal{Q}_B &= \bigl\{Q = (q_1,\dots,q_M)\in\Delta_M \mid \underline{q}_j \le q_j \le \overline{q}_j\bigr\},
\end{align*}
the exact inner minimisation over all product distributions can be written as
\[
\min_{P\in\mathcal{P}_B,\;Q\in\mathcal{Q}_B}
\sum_{i=1}^{N}\sum_{j=1}^{M} a_{ij}\,p_i\,q_j,
\]
where the matrix \(A=(a_{ij})\in\mathbb{R}^{N\times M}\) contains the values \(V^{*}_{\tilde{M}}(s')\) of the successor states \(s'\). Directly optimising this bilinear form is intractable in general.  Instead, we introduce auxiliary variables
\[
h_{ij} \;\approx\; p_iq_j,
\]
for each pair \((i,j)\).  Each \(h_{ij}\) is then constrained by the four McCormick inequalities (Equations~\eqref{eq:mc-lb1}–\eqref{eq:mc-ub2}), which together with the box bounds on \(p_i\) and \(q_j\) enclose exactly the convex hull of each \(\{(p_i,q_j,p_iq_j)\}\).  Finally, to ensure that the auxiliaries \(\{h_{ij}\}_{i,j}\) define a valid probability distribution, we impose the coupling simplex constraint
\[
\sum_{i=1}^{N}\sum_{j=1}^{M} h_{ij} \;=\; 1.
\]
Putting these components together yields the following LP for the inner problem, which is both tractable and tight:
\begin{align*}
\min_{p,q,h}\quad &\sum_{i=1}^{N}\sum_{j=1}^{M} a_{ij}\,h_{ij},\\
\text{subject to}\quad &\underline{p}_i \;\le\; p_i \;\le\; \overline{p}_i,
\quad i=1,\dots,N,\\
                 &\underline{q}_j \;\le\; q_j \;\le\; \overline{q}_j,
\quad j=1,\dots,M,\\
                 &h_{ij} \;\ge\; p_i\,\underline{q}_j + q_j\,\underline{p}_i - \underline{p}_i\,\underline{q}_j,\\
                 &h_{ij} \;\ge\; p_i\,\overline{q}_j + q_j\,\overline{p}_i - \overline{p}_i\,\overline{q}_j,\\
                 &h_{ij} \;\le\; p_i\,\overline{q}_j + q_j\,\underline{p}_i - \underline{p}_i\,\overline{q}_j,\\
                 &h_{ij} \;\le\; p_i\,\underline{q}_j + q_j\,\overline{p}_i - \overline{p}_i\,\underline{q}_j,\\
                 &\sum_{i=1}^{N}\sum_{j=1}^{M} h_{ij} \;=\; 1.
\end{align*}
Each pair \((i,j)\) contributes precisely four McCormick constraints, the box constraints on \(p_i\) and \(q_j\) add \(2N + 2M\) inequalities, and the single global simplex constraint couples all \(h_{ij}\). As a result, the total number of constraints is 
$4NM + 2N + 2M + 1$,
growing linearly in the product support size \(N\cdot M\) and polynomially in the marginal supports.

\subsection{Recursive McCormick Relaxation}

When composing \(n>2\) marginal box-like uncertainty sets 
\(\{\mathcal{P}_B^{k}\subseteq\Delta_{m_k}\}_{k=1}^{n}\), each of the form:
\[
\mathcal{P}^{k}_B = \bigl\{(p^k_1,\dots,p^k_{m_k})\in\Delta_{m_k}
  \;\big|\; \underline p^k_i\le p^k_i\le\overline p^k_i\bigr\},
\]
the inner minimisation becomes the multilinear program:
\[
\min_{P^1\in \mathcal{P}^1_B,\dots,P^n\in \mathcal{P}^n}
\sum_{i_1=1}^{m_1}\cdots\sum_{i_n=1}^{m_n}
a_{i_1\cdots i_n}\,\prod_{k=1}^n p^k_{i_k}.
\]
We follow the standard recursive extension of the McCormick relaxation to multi-linear problems as described by~\citet{DBLP:journals/jgo/RyooS01} and~\citet{DBLP:journals/corr/abs-2207-08955}, which operates in three steps:

\paragraph{Step 1: Introduce auxiliaries.}  
For each index-tuple \((i_1,\dots,i_n)\), we introduce \(n-1\) auxiliary variables replacing each bilinear partial product:
\[
\begin{aligned}
&p_{i_1}^1\,p_{i_2}^2\,\dots\,p_{i_{n-1}}^{n-1}\,p_{i_n}^n
   \\[-10pt]
& \underset{=: h_1}{\rotatebox{270}{$\left.\rule{0pt}{0.5cm}\right\}$}}\\[-10pt]
& \underset{=: h_{n-2}}{\rotatebox{270}{$\left.\rule{0pt}{1.25cm}\right\}$}}\\[-10pt]
& \underset{=: h_{n-1}}{\rotatebox{270}{$\left.\rule{0pt}{1.55cm}\right\}$}}\\[0pt]
\end{aligned}
\]

\paragraph{Step 2: McCormick relaxations.}  
Each bilinear equation \(h \approx u\,v\) (with either \(u=p^1_{i_1},v=p^2_{i_2}\) or \(u=h_{r-1},v=p^{r}_{i_{r}}\)) is relaxed by the four McCormick inequalities (Eqs.~\eqref{eq:mc-lb1}–\eqref{eq:mc-ub2}).  The upper and lower bounds for each auxiliary $s_{r}$ are obtained by simple interval arithmetic:
\begin{align*}
    (\underline{h}_1,\overline{h}_1) &= (\underline{p}_{i_1}^1\underline{p}_{i_2}^2, \overline{p}_{i_1}^1\overline{p}_{i_2}^2), \text{ and}\\
    (\underline{h}_r,\overline{h}_r) &= (\underline{h}_{r-1}\underline{p}_{i_r}^r, \overline{h}_{r-1}\overline{p}_{i_r}^r), \text{ for $r > 1$}.
\end{align*}

\paragraph{Step 3: Linking and objective.}  
Replace the original objective by
\[
\min_{p,h}\quad \sum_{i_1,\dots,i_n} a_{i_1\cdots i_n}\;h^{\,n-1}_{i_1,\dots,i_n},
\]
subject to:
\begin{itemize}
  \item Box constraints \(\underline p^k_i\le p^k_i\le \overline p^k_i\) for all \(k,i\).
  \item McCormick inequalities for each bilinear pair \((u,v)\) defining every \(h^r\).
  \item Global simplex constraint 
    \(\sum_{i_1,\dots,i_n}h^{\,n-1}_{i_1,\dots,i_n}=1\).
\end{itemize}

\paragraph{Complexity.}  
There are \(\prod_{k}m_k\) index‐tuples and \(n-1\) auxiliaries per tuple, yielding \(4(n-1)\prod_{k}m_k\) McCormick inequalities.  Adding \(2\sum_km_k\) box constraints and one simplex equality gives the total number of constraints:
\[
4(n-1)\prod_{k = 1}^n m_k \;+\;2\sum_{k=1}^nm_k\;+\;1
\in \mathcal O\bigl(n \cdot \prod_{k = 1}^n m_k\bigr).
\]
This grows linearly in the product support and exponentially only in the number of factors \(n\), while it is polynomial in the individual support sizes, avoiding a doubly-exponential blow‐up of vertex enumeration.  Furthermore, by propagating bounds via interval arithmetic at each step, the relaxation remains at least as tight as the interval‐arithmetic relaxation.

\section{Vertex Cardinality for $L_1$ and $L_\infty$ Uncertainty Sets}
\label{app:vertexgrwoth}

We demonstrate that for common classes of marginal uncertainty sets, specifically \(L_{1}\), \(L_{\infty}\) balls or box-type sets, the number of vertices can grow exponentially with the support size. As a result, solving the inner optimisation problem in Equation~\eqref{eq:bellman} via explicit vertex enumeration can become infeasible even for moderate support dimensions.

Consider an \(L_\infty\) uncertainty set \(\mathcal{P}\) centred at the uniform nominal distribution
\[
\hat{P} = \left(\tfrac{1}{k}, \dots, \tfrac{1}{k}\right) \in \Delta_k,
\]
with \(k\) even, and radius \(\varepsilon = \tfrac{1}{k}\), that is,
\[
\mathcal{P} = \left\{ P \in \Delta_k \;\middle|\; \|P - \hat{P}\|_\infty \leq \tfrac{1}{k} \right\},
\]
which is a polytope within the \(k\)-dimensional probability simplex and is therefore constrained to the \((k - 1)\)-dimensional hyperplane defined by \(\sum_i p_i = 1\).

For each subset \(I \subseteq \{1, \dots, k\}\) with \(|I| = k/2\), define a distribution \(P^I\) by
\[
p_i^I =
\begin{cases}
\frac{2}{k}, & i \in I, \\
0, & i \notin I.
\end{cases}
\]
It is straightforward to verify that \(P^I \in \mathcal{P}\) for any such \(I\), since these are valid probability distributions and the \(L_\infty\) constraint imposes the bounds \(0 \leq p_i \leq \tfrac{2}{k}\). 

A point in a \((k - 1)\)-dimensional polytope is a vertex if it lies at the intersection of \((k - 1)\) independent active constraints (excluding the affine simplex equality that defines the ambient space). At any distribution \(P^I\), there are:
\begin{enumerate}
  \item \(k/2\) active upper bounds \(p_i = \tfrac{2}{k}\) for \(i \in I\),
  \item \(k/2\) active lower bounds \(p_i = 0\) for \(i \notin I\).
\end{enumerate}
This gives a total of \(k\) active inequality constraints. One can verify that among these \(k\) coordinate-bound constraints, one linear combination corresponds to the simplex constraint, leaving exactly \(k - 1\) linearly independent active inequalities. Therefore, each \(P^I\) is indeed a vertex of \(\mathcal{P}\).

There are \(\binom{k}{k/2}\) such subsets \(I\) with \(|I| = k/2\), and thus at least that many distinct vertices in \(\mathcal{P}\). The quantity \(\binom{k}{k/2}\) is known as the \emph{central binomial coefficient}, and its growth can be accurately described using Stirling's approximation~\citep{Stirling1730,DBLP:books/daglib/0023751} as
\[
\binom{k}{k/2} \in \Theta\left(\frac{2^k}{\sqrt{k}}\right).
\]

Using the same construction, we can show that the \(L_1\) uncertainty set
\[
\mathcal{P} = \left\{ P \in \Delta_k \;\middle|\; \|P - \hat{P}\|_1 \leq 1 \right\}
\]
also contains at least \(\binom{k}{k/2}\) vertices. Note that the maximum \(L_1\) distance between any two probability distributions in \(\Delta_k\) is 2, and $||P^I - \hat{P}||_1 = 1$ for each $I$. Hence, this allows for the same type of extremal distributions constructed above.

Therefore, for both of these common classes of marginal uncertainty sets, explicit vertex enumeration of the resulting product polytope is generally intractable, and we must instead rely on the relaxations introduced in Section~\ref{sec:relaxations}.

We remark that the examples above are not pathological but rather characteristic of marginal uncertainty sets represented by  $L_1$- or $L_\infty$-norm balls. These sets induce combinatorial constraints, i.e., choices of which coordinates hit a bound, and each such choice can give rise to a distinct vertex. This is a natural outcome whenever uncertainty acts coordinate-wise.
However, we also note that the construction above represents an extreme case. The number of vertices can  be significantly smaller, for example when the nominal distribution lies close to a vertex of the probability simplex.

\section{Overapproximating $L_1$ Balls as Boxes}
\label{app:l1box}
McCormick envelopes are constructed using upper and lower bounds on each component of the marginal distributions, and thus naturally accommodate the product of box-type uncertainty sets derived from confidence intervals, such as the exact Clopper–Pearson interval~\citep{ClopperPearson1934} discussed in Section~\ref{sec:learning}. Our experiments show that McCormick envelopes are effective in circumventing the exponential blow-up in the number of vertices when analyzing the product of such polytopes. 
However, they do not directly extend to uncertainty sets in the form of \(L_1\) balls derived from Weissman’s inequality~\citep{WeissmanEtAl2003}. One might attempt to apply McCormick envelopes to these \(L_1\) sets by computing upper and lower bounds on each component of the marginal distribution under the \(L_1\) constraint. Yet, as we show below, the tightest such bounds coincide precisely with those obtained from Hoeffding's inequality~\citep{Hoeffding1994}, which are known to be less tight than the exact Clopper–Pearson bounds.

As a result, McCormick envelopes offer no practical advantage for \(L_1\) uncertainty sets, and we must instead rely on the radii-sum result in Theorem~\ref{thm:radiisum} to enable efficient computation under the product uncertainty. This is the same approach employed by~\citet{Strehl2007} in his PAC analysis for model-based reinforcement learning in factored MDPs.

Let \(\mathcal{P}\) be an \(L_1\) uncertainty set defined as
\[
\mathcal{P}
= \left\{\,P' \in \Delta_a \;\middle|\;
\|P' - \hat{P}\|_{1} \le \varepsilon \right\},
\]
centred around the nominal distribution \(\hat{P} = (\hat{p}_1, \dots, \hat{p}_a) \in \Delta_a\), where \(\varepsilon\) is given by~\citet{WeissmanEtAl2003} as
\[
\varepsilon = \sqrt{\frac{2\bigl[\ln(2^{a} - 2) - \ln(\delta)\bigr]}{n}},
\]
with \(a\) denoting the support size and \(n\) the number of observed samples.

For any \(P' \in \mathcal{P}\) and for each component \(p'_i\), the maximum possible deviation from the nominal value \(\hat{p}_i\) is \(\varepsilon/2\). This is because, to satisfy the \(L_1\) constraint, one can allocate at most \(\varepsilon/2\) of additional probability mass to a single component, compensated by removing \(\varepsilon/2\) from the others, preserving the total mass of $1$ and ensuring that the sum of deviations remains bounded by \(\varepsilon\). This yields:
\[
\frac{\varepsilon}{2} = \sqrt{\frac{\ln(2^{a} - 2) - \ln(\delta)}{2n}}.
\]
Since the bound \(\varepsilon\) increases with the support size \(a\), the smallest possible deviation arises in the case of \(a = 2\) successors, leading to:
\begin{align*}
\frac{\varepsilon}{2} &= \sqrt{\frac{\ln(2) - \ln(\delta)}{2n}} \\
&= \sqrt{\frac{\ln\left(\frac{2}{\delta}\right)}{2n}},
\end{align*}
which matches the classical confidence bound for the binomial distribution derived from Hoeffding's inequality~\citep{Hoeffding1994}. This bound is known to be no tighter than the exact confidence interval given by~\citet{ClopperPearson1934}, which we use in our construction of box-type uncertainty sets in Section~\ref{sec:learning}.
Therefore, the tightest possible component-wise upper and lower bounds that can be derived from an \(L_1\) constraint are never sharper than those obtained directly from exact confidence intervals. As a result, there is no practical advantage to combining \(L_1\) uncertainty sets with McCormick envelopes.

\section{Detailed Benchmark Environments}
\label{app:benchmarks}

We provide details on the benchmark environments employed in our experiments and the used hyperparameters. For learning benchmarks, we may select separate instances from those used in the solution case studies, since the rf-MDP must be re‐solved each time the sampling policy is updated. Specifically, whenever any factor’s sample count doubles (cf.\ Section~\ref{sec:experiments}), we recompute the optimal optimistic policy during the rollout of a total of $10^6$ fixed‐length trajectories. All our used environments and case studies are publicly available\footnote{\url{https://zenodo.org/records/17580296}}.

\subsection{Aircraft Collision Avoidance}
The aircraft collision avoidance environment is a version of the family of models introduced in~\citet{kochenderfer2015}. We consider an \(N \times M\) grid, where two aircraft, one controlled by our agent and one adversarial, fly toward each other. At each time step, both pilots may choose to fly straight, ascend, or descend, with these actions succeeding independently with probabilities \(p\) and \(q\), respectively. The agent's objective is to reach the opposite end of the grid without colliding with the adversarial aircraft, which manoeuvres arbitrarily. A collision is defined as entering a specified radius around the adversarial aircraft. The objective value is the probability of reaching the goal zone without collision. The two aircraft form the factors of the factored MDP, and evolve independently until a collision occurs. For the solving benchmarks, we set $(N,M)=(20,24)$. For the learning benchmarks, we employ environments with $(N,M)=(15,15)$ and sample trajectories of fixed length $l = 15$.

\subsection{System Administrator}
The System Administrator (or SysAdmin) domain is a standard benchmark in factored MDPs~\cite{DBLP:journals/jair/GuestrinKPV03}. In this setting, an administrator manages a network of \(N\) computers, each connected to a subset of the others. We consider the bidirectional ring topology, where each computer is connected to its two immediate neighbors, forming a closed loop. 
At each time step, a computer can be in one of two states: running or failed. The probability that a running computer fails depends on whether its connected machines are currently in the failure state. As a result, the computers form the factors of the factored MDP, with dependencies only between adjacent nodes. 
The administrator may choose to repair one machine per time step, which then returns to the running state with high probability. The administrator receives a reward at each step that increases with the number of machines currently running. The objective is to maximise the expected cumulative reward over a fixed time horizon $T$. For the solving benchmarks, we consider $N=10$ machines with a time horizon of $T=15$. For the learning benchmarks, we use $T=10$ and sample trajectories of that length.

\subsection{Frozen Lake} The Frozen Lake environment is a standard benchmark from OpenAI Gym~\cite{DBLP:journals/corr/gym}. We consider an \(N \times N\) frozen lake grid, where some cells contain holes in the ice. The agent can move in the four cardinal directions and aims to reach a designated goal cell without falling into any of the holes. Due to the slippery surface, there is a probability at each step that the agent ends up in a nearby cell rather than the one intended.
To model this as a factored MDP, we extend the environment into a multi-agent setting: we introduce a second agent, and the task is now to manoeuvre both agents so that they each reach their goal without falling into a hole or colliding with each other. The two agents act as the factors of the factored MDP and move independently in the environment, except when a collision occurs. The objective is to minimise the expected number of steps required for both agents to reach their respective goal cells. For the solving benchmarks we use a grid of \(N=15\) size, whereas for the learning benchmarks we use \(N=8\),  and we sample trajectories of fixed length $l = 100$.  

\subsection{Stock Trading} The stock-trading domain~\citep{DBLP:conf/aaai/StrehlDL07} simulates a stock market with \(N\) economic sectors, each containing \(M\) individual stocks. At each time step, the agent can choose to buy or sell an entire sector, thereby either owning or not owning all the stocks within that sector in the next step.
Stocks can at each step be rising or falling. The probability that a stock rises depends on whether other stocks in the same sector were rising in the previous time step. As a result, the sectors form the factors of the factored MDP and evolve independently from one another. The agent receives a positive reward for each rising stock it owns, and a negative reward for each owned stock that falls. No reward is gained or lost from stocks in sectors the agent does not own. The objective is to maximise the cumulative reward over a fixed time horizon \(T\). For the solving benchmarks, we use \(N=3\) sectors with \(M=2\) stocks each and \(T=20\). For the learning benchmarks, we set \(N=2\), \(M=2\), and \(T=10\). We sample trajectories for the full time horizon length $l = 10$.

\subsection{Drone Delivery}
The drone delivery problem~\cite{DBLP:journals/sttt/BadingsCJJKT22,DBLP:conf/tacas/SchnitzerAP25} is a multi-agent factored MDP in which \(N\) autonomous drones navigate a shared 3D environment to deliver payloads to a designated target zone while avoiding both static obstacles and each other. At each time step, every drone may choose one of six actions: move north, south, east, west, ascend, or descend. Due to wind, the dynamics are subject to stochastic disturbances and the drone may drift. 
Each drone \(i\) constitutes a factor in the factored MDP: its local state is its \((x,y,z)\) coordinate, and its transition kernel depends only on its own current state and chosen action. The objective value is the probability of safely reaching the target zone without crashing, over an infinite time horizon. We consider a version with \(N=2\) drones as the state factors. For the solving benchmarks we use an \(8\times8\times8\) environment, and for the learning benchmarks we use dimensions \(5\times5\times5\) and learn from sampled trajectories of length $l = 50$.  

\subsection{Herman’s Self-Stabilising Protocol}
The Herman self-stabilising protocol~\cite{DBLP:journals/ipl/Herman90} is a randomised synchronous algorithm for achieving self-stabilisation in a unidirectional token-ring of \(N\) identical processes. Each process \(i\) maintains a Boolean variable \(x_i\in\{0,1\}\), and initially an arbitrary (odd) number of tokens may be present, each token being indicated by a process having the same bit as its counterclockwise neighbour. At each time step, every process that holds a token (i.e.\ \(x_i = x_{i-1}\)) flips its bit with probability \(r\in(0,1)\) (thereby passing the token clockwise) and retains its bit with probability \(1-r\). Processes without a token keep their bit unchanged. If a process retains its token and simultaneously receives another token from its neighbour, both tokens annihilate.  
We model this as a factored MDP. The global state factorises over the \(N\) processes, each of which evolves conditionally on its own bit and that of its counter-clockwise neighbour. 
The processes are \emph{stable} when a single token is present, the objective value is the expected number of steps until stabilisation is attained. We consider a version with $N= 11$ parallel processes.

\subsection{Chain}
The chain benchmark~\cite{DBLP:conf/ewrl/Araya-LopezBTC11} comprises \(N\) independent chains, each with states \(\{1,\dots,M\}\), constituting the factors. At every time step, a chain in state \(i < M-1\) either advances to \(i+1\) with probability $p$ or resets to state~1 with probability $1-p$. From the penultimate state \(M-1\), it either moves to the terminal state \(M\) or, upon failure, falls back uniformly to one of the earlier states \(\{1,\dots,M-2\}\). Once a chain reaches state \(M\), it remains there indefinitely. The objective is the expected number of steps for all chains to reach the final state \(M\). In our experiments, we consider a version with \(N=2\) chains and \(M=10\) states per chain.  

\section{Extended Experiments}
\label{app:experiments}

In the following, we provide extended results for both solving robust factored MDPs and learning robust policies. All experiments were performed on an Intel Xeon Gold (2.50 GHz, 40 cores) with 128 GB of RAM.

\paragraph{Solving rf-MDPs}Table~\ref{tab:case_study_extended} presents solving benchmarks across varying $L_\infty$ uncertainty radii \(\varepsilon\) around a nominal transition kernel, extending Section~\ref{sec:experiments}, which focused on \(\varepsilon=0.025\). As \(\varepsilon\) increases, the relative gap between the exact rf-MDP solution (via vertex enumeration) and the interval-arithmetic relaxation can rise significantly, up to 900\% in the most extreme of the considered cases, highlighting the growing conservatism of  the interval-arithmetic relaxation. In contrast, the McCormick relaxation remains exact for all tested radii, demonstrating its tightness even under elevated epistemic uncertainty.  

\paragraph{Robust Policy Learning}  
Figure~\ref{fig:learning_extended} displays results for all learning benchmarks, including the Stock Trading domain omitted from Figure~\ref{fig:learningresults}. Note that, in the Frozen Lake domain, the agent aims to minimise the objective, so the y-axis is inverted to reflect better performance at lower values. For each method and domain, we also report the total runtime required to sample and process trajectories. This runtime accounts for (i) sampling trajectories, (ii) recomputing the optimistic sampling policy, by solving the current rf-MDP under the most favorable environment policy, (iii) synthesising the optimal robust policy, and (iv) evaluating its performance on the hidden true MDP. 
Across all domains and methods, runtimes remain comparable, showing that the superior tightness of McCormick relaxations incurs negligible computational overhead.

We do not include the Chain and Herman domains in the learning benchmarks. These MDPs offer only a single action per state and evolve purely stochastically, making them unsuitable for exploration and policy improvement, though they remain instructive for solution‐quality evaluation.

In Table~\ref{tab:comparison}, we explicitly compare the sample efficiency and runtimes of the different methods by showing how many fixed-length trajectories each method requires to achieve the performance guarantee attained by the slowest method after processing the full set of $10^6$ trajectories. In our experiments, the slowest method was always the flat MDP learner that does not exploit the factored structure. The results demonstrate that factored learning with the McCormick relaxation is significantly more sample-efficient than the other methods, requiring orders of magnitude fewer samples than flat and $L_1$-based learning, and typically less than half the number of samples needed compared to using the interval-arithmetic relaxation. However, the runtimes remain comparable, although the interval-arithmetic relaxation, can be solved very efficiently using bisection techniques. This is due to the substantially reduced sample requirements.

\begin{landscape}
\begin{table}[!htbp]
\centering
\resizebox{0.9\linewidth}{!}{%
  \begin{tabular}{c c c c | c c | c c c | c c c}
    \toprule
    \multirow{2}{*}{\textbf{Domain}} 
      & \multirow{2}{*}{$\bm{|S|}$} 
      & \multirow{2}{*}{$\bm{|T|}$} 
      & \multirow{2}{*}{$\bm{\varepsilon}$}
      & \multicolumn{2}{c}{\textbf{Vertex Enumeration}}
      & \multicolumn{3}{c}{\textbf{Interval-Arithmetic}}
      & \multicolumn{3}{c}{\textbf{McCormick}} \\
    \cmidrule(lr){5-6} \cmidrule(lr){7-9} \cmidrule(lr){10-12}
    & & & 
      & \textbf{Robust Value}  & \textbf{Time [s]}
      & \textbf{Robust Value}  & \textbf{Rel.\ Gap} & \textbf{Time [s]}
      & \textbf{Robust Value}  & \textbf{Rel.\ Gap} & \textbf{Time [s]} \\
    \midrule
\multirow{6}{*}{\text{Aircraft ($\uparrow$)}} & \multirow{6}{*}{11153} & \multirow{6}{*}{1262099} & $0.01$ & $0.85$ & $2673.7$ & $0.83$ & $3$\% & $5.3$ & $0.85$ & $0$\% & $45.2$ \\
 &  &  & $0.02$ & $0.77$ & $3095.0$ & $0.71$ & $7$\% & $5.2$ & $0.77$ & $0$\% & $45.4$ \\
 &  &  & $0.025$ & $0.73$ & $2535.8$ & $0.65$ & $11$\% & $6.1$ & $0.73$ & $0$\% & $43.7$ \\
 &  &  & $0.03$ & $0.68$ & $2568.9$ & $0.57$ & $15$\% & $4.4$ & $0.68$ & $0$\% & $41.3$ \\
 &  &  & $0.04$ & $0.56$ & $1642.1$ & $0.42$ & $25$\% & $5.3$ & $0.56$ & $0$\% & $42.2$ \\
 &  &  & $0.1$ & $0.04$ & $1011.1$ & $0.004$ & $900$\% & $4.3$ & $0.04$ & $0$\% & $45.5$ \\
\midrule
\multirow{6}{*}{\text{Chain ($\downarrow$)}} & \multirow{6}{*}{100} & \multirow{6}{*}{3136} & $0.01$ & $266.45$ & $106.1$ & $299.11$ & $12$\% & $0.5$ & $266.45$ & $0$\% & $5.1$ \\
 &  &  & $0.02$ & $308.00$ & $699.5$ & $392.14$ & $27$\% & $0.6$ & $308.00$ & $0$\% & $7.0$ \\
 &  &  & $0.025$ & $331.34$ & $778.1$ & $451.28$ & $36$\% & $0.6$ & $331.34$ & $0$\% & $7.6$ \\
 &  &  & $0.03$ & $356.66$ & $746.8$ & $521.27$ & $46$\% & $0.7$ & $356.66$ & $0$\% & $7.8$ \\
 &  &  & $0.04$ & $413.98$ & $1054.5$ & $703.70$ & $70$\% & $1.0$ & $413.98$ & $0$\% & $9.7$ \\
 &  &  & $0.1$ & $1019.70$ & $5172.9$ & $6355.96$ & $523$\% & $4.8$ & $1019.70$ & $0$\% & $26.0$ \\
\midrule
\multirow{6}{*}{\text{Drone ($\uparrow$)}} & \multirow{6}{*}{262144} & \multirow{6}{*}{21694720} & $0.01$ & $0.72$ & $198.0$ & $0.70$ & $3$\% & $166.9$ & $0.72$ & $0$\% & $124.6$ \\
 &  &  & $0.02$ & $0.70$ & $1913.8$ & $0.65$ & $7$\% & $104.7$ & $0.70$ & $0$\% & $202.6$ \\
 &  &  & $0.025$ & $0.69$ & $2125.8$ & $0.63$ & $9$\% & $90.2$ & $0.69$ & $0$\% & $190.7$ \\
 &  &  & $0.03$ & $0.68$ & $2269.7$ & $0.60$ & $11$\% & $92.8$ & $0.68$ & $0$\% & $196.0$ \\
 &  &  & $0.04$ & $0.66$ & $2353.6$ & $0.55$ & $16$\% & $94.7$ & $0.66$ & $0$\% & $195.3$ \\
 &  &  & $0.1$ & $0.51$ & $599.0$ & $0.24$ & $53$\% & $183.1$ & $0.51$ & $0$\% & $197.9$ \\
\midrule
\multirow{6}{*}{\text{Herman ($\downarrow$)}} & \multirow{6}{*}{2048} & \multirow{6}{*}{177148} & $0.01$ & $18.13$ & $9.3$ & $19.10$ & $5$\% & $2.5$ & $18.13$ & $0$\% & $8.0$ \\
 &  &  & $0.02$ & $19.74$ & $10.3$ & $22.06$ & $12$\% & $2.7$ & $19.74$ & $0$\% & $7.4$ \\
 &  &  & $0.025$ & $20.64$ & $11.0$ & $23.82$ & $15$\% & $2.8$ & $20.64$ & $0$\% & $8.1$ \\
 &  &  & $0.03$ & $21.61$ & $11.4$ & $25.82$ & $20$\% & $3.3$ & $21.61$ & $0$\% & $8.9$ \\
 &  &  & $0.04$ & $23.78$ & $12.4$ & $30.66$ & $29$\% & $3.8$ & $23.78$ & $0$\% & $9.0$ \\
 &  &  & $0.1$ & $48.73$ & $24.7$ & $124.21$ & $155$\% & $11.8$ & $48.73$ & $0$\% & $19.2$ \\
\midrule
\multirow{6}{*}{\text{Frozen Lake ($\downarrow$)}} & \multirow{6}{*}{50625} & \multirow{6}{*}{1866556} & $0.01$ & $194.58$ & $83.5$ & $202.49$ & $4$\% & $56.1$ & $194.58$ & $0$\% & $55.6$ \\
 &  &  & $0.02$ & $208.29$ & $1023.3$ & $227.19$ & $9$\% & $66.0$ & $208.29$ & $0$\% & $100.8$ \\
 &  &  & $0.025$ & $216.01$ & $1018.4$ & $242.05$ & $12$\% & $67.7$ & $216.01$ & $0$\% & $105.9$ \\
 &  &  & $0.03$ & $224.41$ & $1082.9$ & $259.12$ & $15$\% & $77.6$ & $224.41$ & $0$\% & $113.3$ \\
 &  &  & $0.04$ & $243.60$ & $1218.9$ & $302.35$ & $24$\% & $99.8$ & $243.60$ & $0$\% & $130.5$ \\
 &  &  & $0.1$ & $520.80$ & $1687.3$ & $2797.57$ & $437$\% & $1282.7$ & $520.80$ & $0$\% & $316.3$ \\
\midrule
\multirow{6}{*}{\text{Stock Trading ($\uparrow$)}} & \multirow{6}{*}{12481} & \multirow{6}{*}{5362624} & $0.01$ & $30.44$ & $34.8$ & $26.42$ & $13$\% & $16.8$ & $30.44$ & $0$\% & $35.4$ \\
 &  &  & $0.02$ & $27.02$ & $66.5$ & $20.23$ & $25$\% & $14.9$ & $27.02$ & $0$\% & $64.3$ \\
 &  &  & $0.025$ & $25.43$ & $67.6$ & $17.60$ & $31$\% & $16.0$ & $25.43$ & $0$\% & $67.5$ \\
 &  &  & $0.03$ & $23.91$ & $66.0$ & $15.22$ & $36$\% & $17.5$ & $23.91$ & $0$\% & $64.6$ \\
 &  &  & $0.04$ & $21.13$ & $65.0$ & $11.17$ & $47$\% & $17.4$ & $21.13$ & $0$\% & $67.3$ \\
 &  &  & $0.1$ & $8.81$ & $101.5$ & $1.91$ & $78$\% & $13.2$ & $8.81$ & $0$\% & $104.6$ \\
\midrule
\multirow{6}{*}{\text{SysAdmin ($\uparrow$)}} & \multirow{6}{*}{15873} & \multirow{6}{*}{9332587} & $0.01$ & $54.66$ & $68.4$ & $52.77$ & $3$\% & $38.9$ & $54.66$ & $0$\% & $67.5$ \\
 &  &  & $0.02$ & $51.94$ & $68.5$ & $48.49$ & $7$\% & $33.3$ & $51.94$ & $0$\% & $66.0$ \\
 &  &  & $0.025$ & $50.70$ & $66.7$ & $46.66$ & $8$\% & $34.1$ & $50.70$ & $0$\% & $64.1$ \\
 &  &  & $0.03$ & $49.54$ & $65.7$ & $44.98$ & $9$\% & $34.5$ & $49.54$ & $0$\% & $66.0$ \\
 &  &  & $0.04$ & $47.43$ & $66.7$ & $42.11$ & $11$\% & $34.3$ & $47.43$ & $0$\% & $63.8$ \\
 &  &  & $0.1$ & $39.66$ & $59.2$ & $33.02$ & $17$\% & $31.6$ & $39.66$ & $0$\% & $59.7$ \\
\bottomrule
  \end{tabular}%
}
\caption{Extended results for solution benchmarks. Arrows ($\uparrow$/$\downarrow$) indicate optimisation directions. $|S|$ and $|T|$ denote the number of states and transitions, and $\varepsilon$ is the added $L_\infty$ uncertainty radius. The relative gap is $|V_{\mathrm{VE}} - V_{R}|/V_{R}$, where $V_{\mathrm{VE}}$ and $V_{R}$ are the robust results from vertex enumeration and relaxations.}
\label{tab:case_study_extended}
\end{table}

\end{landscape}

\begin{figure*}[htbp]
  \centering
    \begin{subfigure}[b]{\textwidth}
    \centering
    \includegraphics[width=0.8\textwidth]{figures/legend.pdf }
  \end{subfigure}
  \\[1em] 
  \begin{subfigure}[b]{0.3\textwidth}
    \centering
    \begin{subfigure}[b]{\textwidth}
      \includegraphics[width=\textwidth]{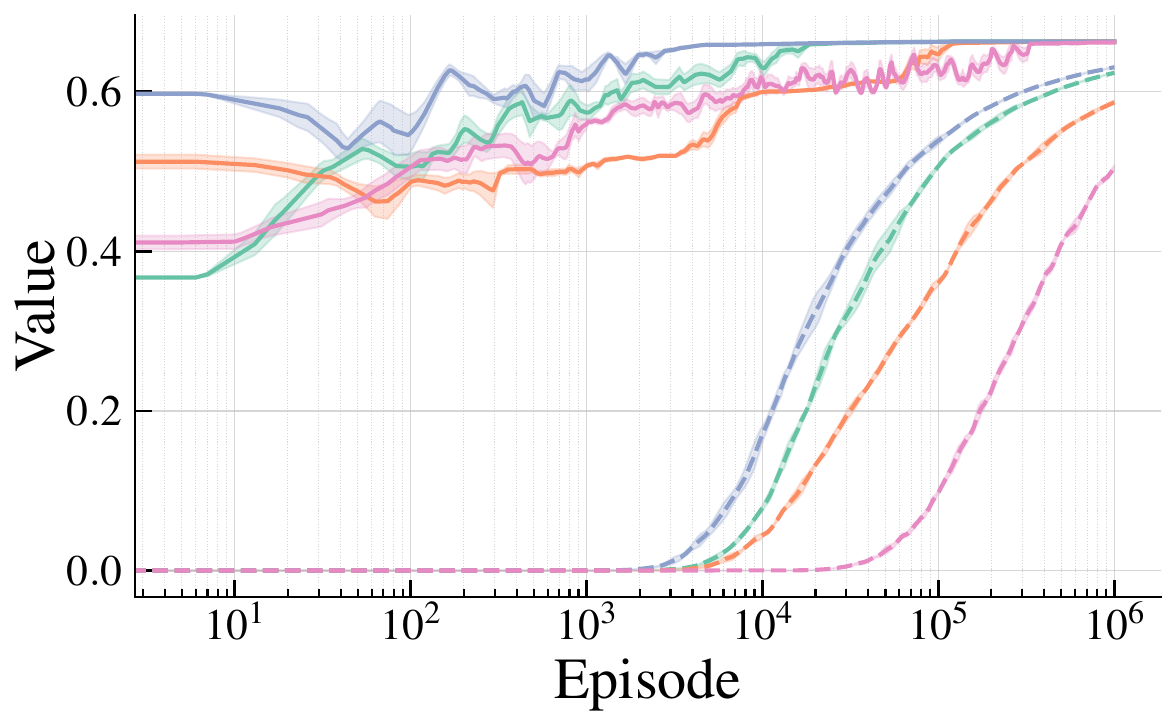}
      \label{fig:1a}
    \end{subfigure}
    \begin{subfigure}[b]{\textwidth}
      \includegraphics[width=\textwidth]{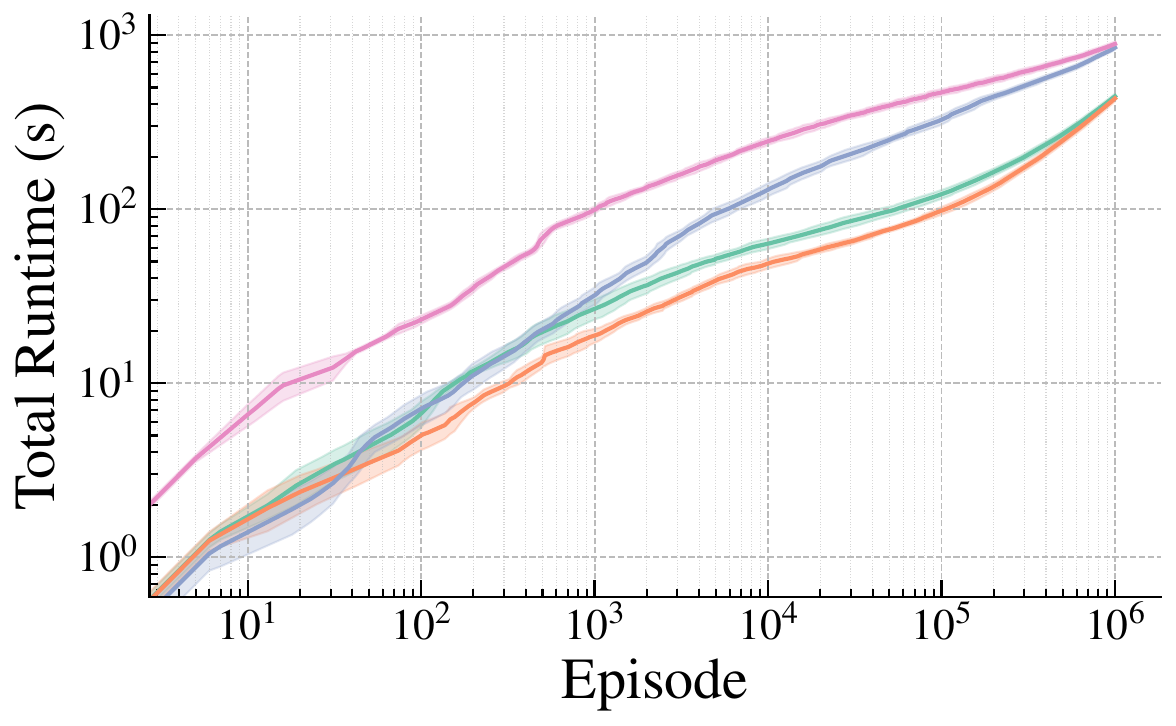}
      \label{fig:1b}
    \end{subfigure}
    \caption{Aircraft}
    \label{fig:stack1}
  \end{subfigure}
  \hspace{1em}
  \begin{subfigure}[b]{0.3\textwidth}
    \centering
    \begin{subfigure}[b]{\textwidth}
      \includegraphics[width=\textwidth]{figures/Experiments/LAKE_SWARM/eps=0.05,N=8,M=8,p=0.2/LAKE_SWARM_eps=0.05,N=8,M=8,p=0.2.pdf}
      \label{fig:2a}
    \end{subfigure}
    \begin{subfigure}[b]{\textwidth}
      \includegraphics[width=\textwidth]{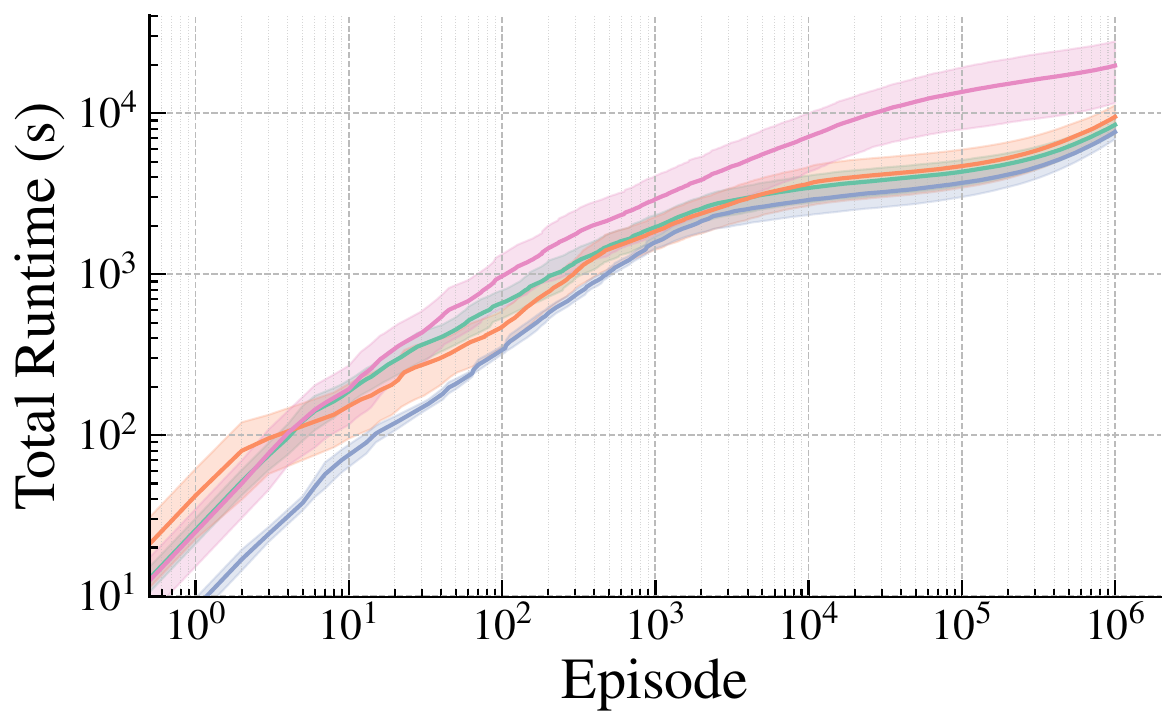}
      \label{fig:2b}
    \end{subfigure}
    \caption{Frozen Lake}
    \label{fig:stack2}
  \end{subfigure}
  \hspace{1em}
  \begin{subfigure}[b]{0.3\textwidth}
    \centering
    \begin{subfigure}[b]{\textwidth}
      \includegraphics[width=\textwidth]{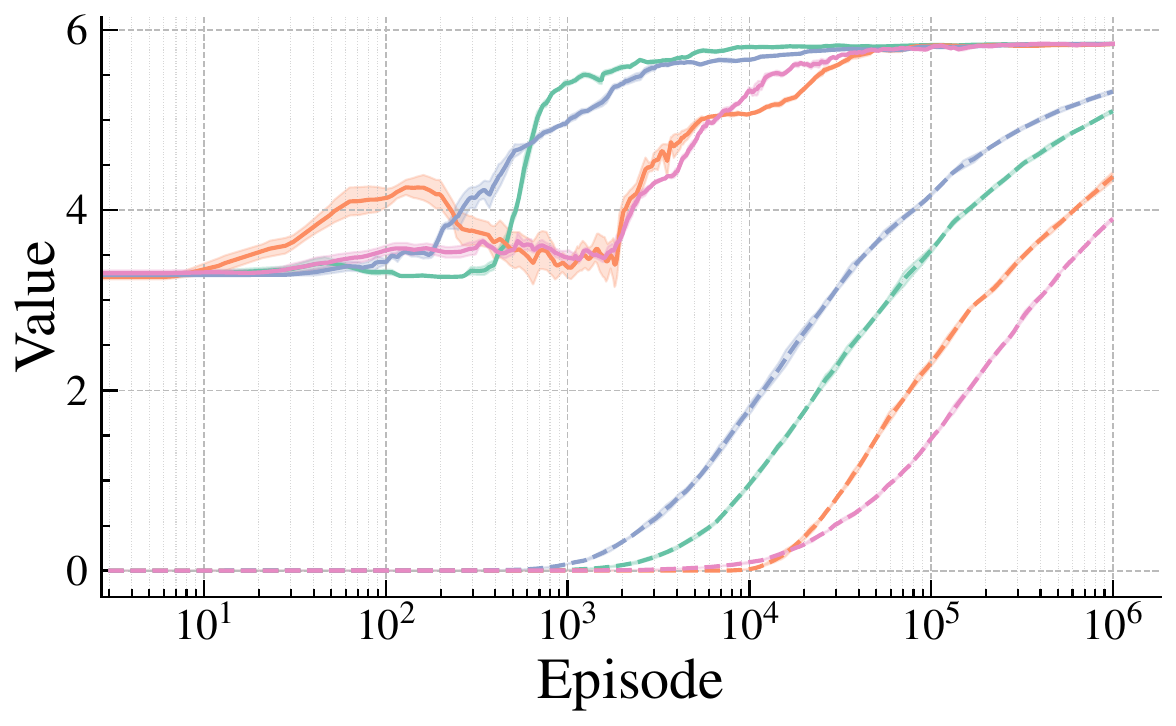}
      \label{fig:3a2}
    \end{subfigure}
    \begin{subfigure}[b]{\textwidth}
      \includegraphics[width=\textwidth]{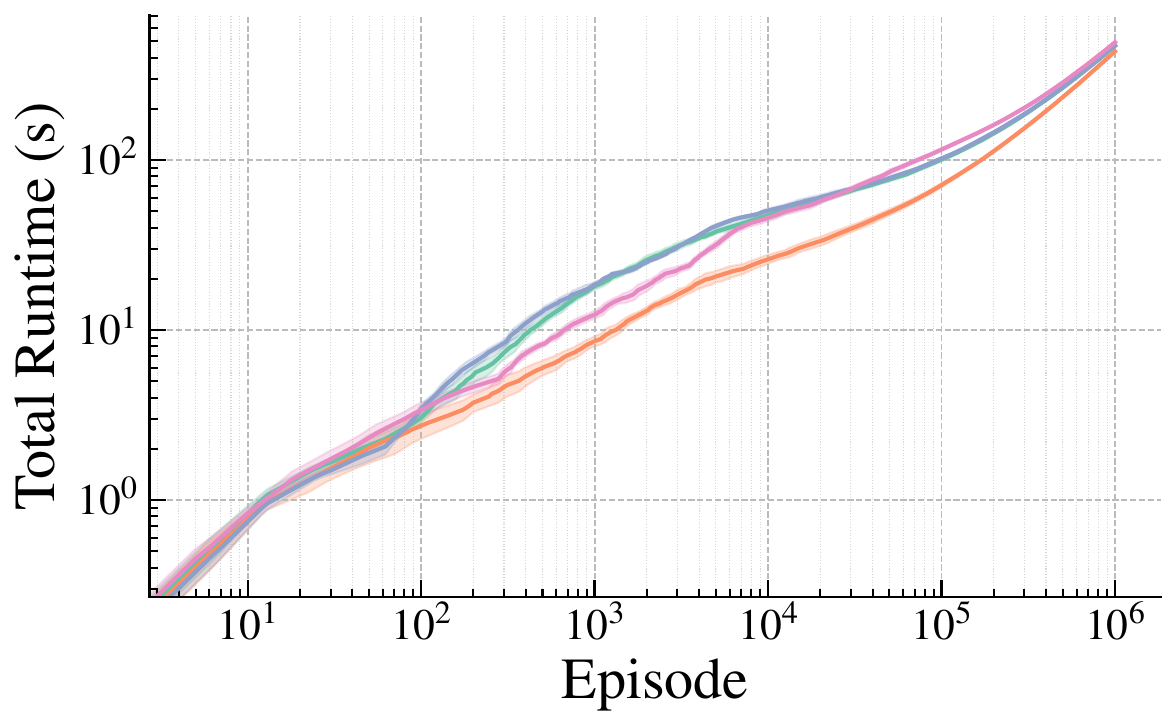}
      \label{fig:3b2}
    \end{subfigure}
    \caption{Stock Trading}
    \label{fig:stack42}
  \end{subfigure}

  \vspace{2ex}

  \begin{subfigure}[b]{0.3\textwidth}
    \centering
    \begin{subfigure}[b]{\textwidth}
      \includegraphics[width=\textwidth]{figures/Experiments/SYSADMIN/T=8,N=10,p0=0.1,p1=0.6/SYSADMIN_T=8,N=10,p0=0.1,p1=0.6.pdf}
      \label{fig:3a}
    \end{subfigure}
    \par\bigskip
    \begin{subfigure}[b]{\textwidth}
      \includegraphics[width=\textwidth]{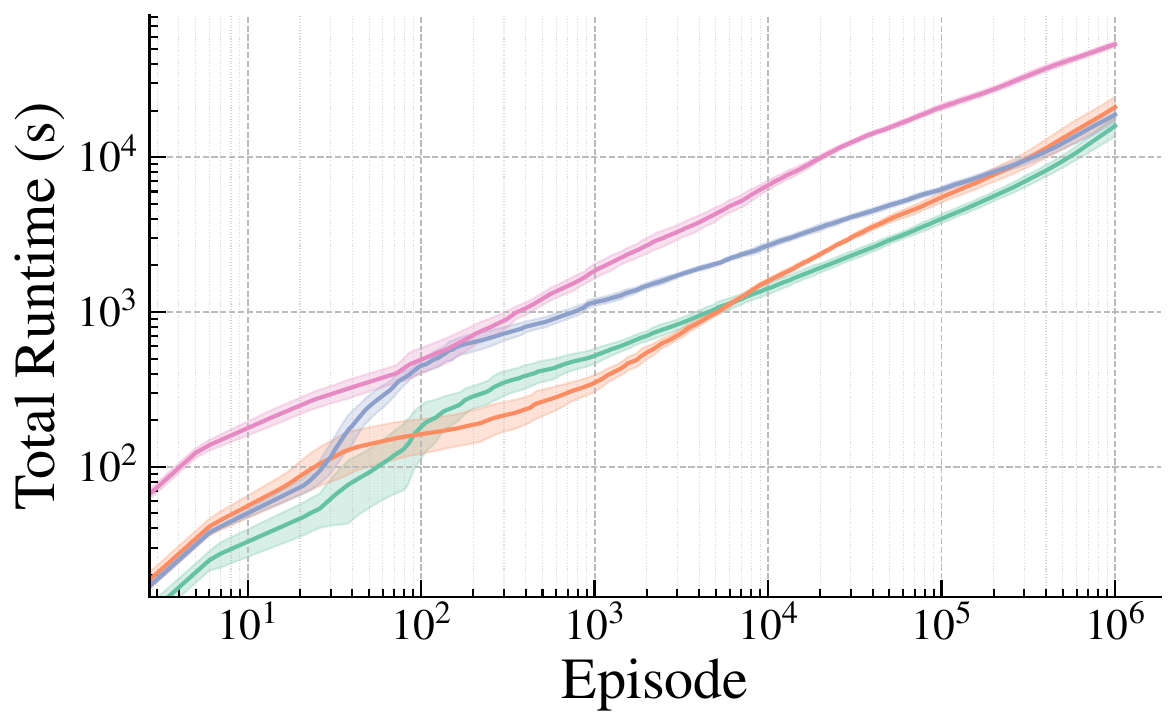}
      \label{fig:3b}
    \end{subfigure}
    \caption{SysAdmin}
    \label{fig:stack3}
  \end{subfigure}
   \hspace{2em}
  \begin{subfigure}[b]{0.3\textwidth}
    \centering
    \begin{subfigure}[b]{\textwidth}
      \includegraphics[width=\textwidth]{figures/Experiments/DRONE_MULTI/MAXX=5,MAXY=5,MAXZ=5,p=0.3,eps=0.029/DRONE_MULTI_MAXX=5,MAXY=5,MAXZ=5,p=0.3,eps=0.029.pdf}
      \label{fig:4a}
    \end{subfigure}
    \par\bigskip
    \begin{subfigure}[b]{\textwidth}
      \includegraphics[width=\textwidth]{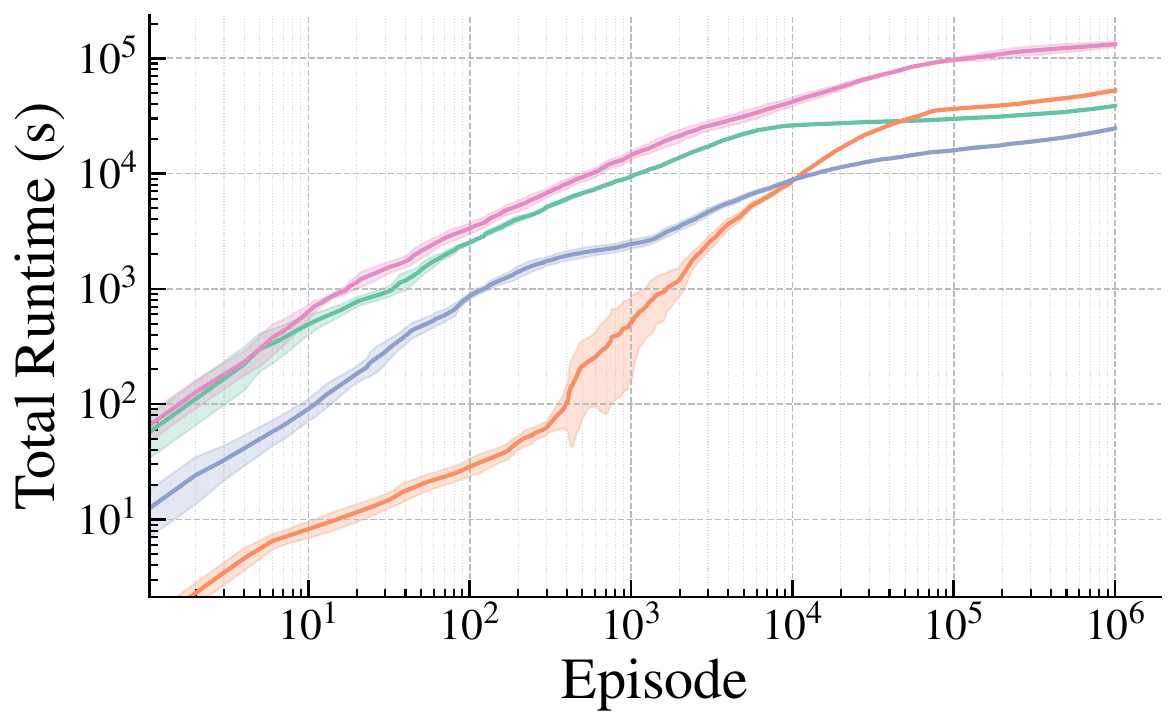}
      \label{fig:4b}
    \end{subfigure}
    \caption{Drone}
    \label{fig:stack4}
  \end{subfigure}

  \caption{Extended results for robust policy learning. For each benchmark, the upper plots compare performance and guarantees, and the lower plots show the total runtimes of each method.}
  \label{fig:learning_extended}
\end{figure*}

\begin{table*}[!htbp]
\centering
\resizebox{0.8\linewidth}{!}{%
  \begin{tabular}{c c | cc | cc | cc | cc}
    \toprule
    \multirow{2}{*}{\textbf{Domain}} & \multirow{2}{*}{\textbf{Guarantee}}
      & \multicolumn{2}{c}{\textbf{Flat Learning}}
      & \multicolumn{2}{c}{\textbf{$L_1$}}
      & \multicolumn{2}{c}{\textbf{Interval-Arithmetic}}
      & \multicolumn{2}{c}{\textbf{McCormick}} \\
    \cmidrule(lr){3-4} \cmidrule(lr){5-6} \cmidrule(lr){7-8} \cmidrule(lr){9-10}
    & & \textbf{Time (s)} & \textbf{Samples}
      & \textbf{Time (s)} & \textbf{Samples}
      & \textbf{Time (s)} & \textbf{Samples}
      & \textbf{Time (s)} & \textbf{Samples} \\
    \midrule
    Aircraft & 0.50 & 872 & $10^6$ & 200 &  $3 \cdot 10^5$ & 130 &  $10^5$ & 290 &  $6 \cdot 10^4$ \\
    Frozen Lake & 72.84 & 9591 & $10^6$ & 6198 &  $1.2 \cdot 10^5$ & 4406 &  $3.6 \cdot 10^4$ & 3318 &  $1.6 \cdot 10^4$ \\
    Stock Trading & 3.90 & 493 & $10^6$ & 237 &  $5.5 \cdot 10^5$ & 142 &  $1.5 \cdot 10^5$ & 85 &  $7 \cdot 10^4$ \\
    SysAdmin & 35.34 & 51634 & $10^6$ & 14484 &  $7.5 \cdot 10^5$ & 2521 &  $4.5 \cdot 10^4$ & 3108 &  $1.3 \cdot 10^4$ \\
    Drone & 0.07 & 123357 & $10^6$ & 8259 &  $9.6 \cdot 10^3$ & 24416 &  $2.1 \cdot 10^4$ & 2555 &  $1.3 \cdot 10^3$ \\
    \bottomrule
  \end{tabular}%
}
\caption{Comparison of sample efficiency and runtime across methods required to reach the guarantee achieved by the slowest method after processing the full set of trajectories.}
\label{tab:comparison}
\end{table*}

\end{document}